\documentclass{article}

\usepackage{microtype}
\usepackage{graphicx}
\usepackage{subfigure}
\usepackage{amsmath,amssymb,amsthm,latexsym}
\usepackage{thm-restate} 
\usepackage{booktabs} 
\usepackage{multirow}

\usepackage{float}

\usepackage{hyperref}

\usepackage{booktabs}
\usepackage[normalem]{ulem}
\useunder{\uline}{\ul}{}

\usepackage{multicol}
\usepackage{ifthen}
\usepackage{thmtools}
\usepackage{tabularx}
\usepackage{xspace}
\usepackage{bbm}
\usepackage[normalem]{ulem}
\usepackage{hyperref}
\usepackage{float}
\usepackage{xcolor}[table]
\usepackage{tikz}
\usepackage{wrapfig}
\newcommand{\tikzxmark}{%
\tikz[scale=0.23] {
    \draw[line width=0.7,line cap=round] (0,0) to [bend left=6] (1,1);
    \draw[line width=0.7,line cap=round] (0.2,0.95) to [bend right=3] (0.8,0.05);
}}

\usepackage{rotating}
\newtheorem{thm}{Theorem}

\newtheorem{lem}{Lemma}
\newtheorem{cor}{Corollary}
\newtheorem{example}{Example}
\newtheorem{prop}{Proposition}
\newtheorem{obs}{Observation}
\newtheorem{asum}{Assumption}

\usepackage[most]{tcolorbox}
\usepackage{wrapfig}

\usepackage{listings}
\definecolor{githubblue}{RGB}{49,46,138}
\definecolor{schemegreen}{RGB}{15,137,15}
\definecolor{operator}{RGB}{0,.3,.7}
\definecolor{azure}{rgb}{0,.3,.7}
\definecolor{comment_color}{rgb}{0.24, 0.51, 0.51}
\definecolor{applegreen}{rgb}{0.55, 0.71, 0.0}
\definecolor{deepcarmine}{rgb}{0.66, 0.13, 0.24}

\newcommand{\xvar}{\ensuremath{\textcolor{deepcarmine}{\mathsf{x}}}}
\newcommand{\yvar}{\ensuremath{\textcolor{deepcarmine}{\mathsf{y}}}}

\newcommand{\probability}{\ensuremath{\text{P}}}

\newcommand{\progvariable}{\ensuremath{\mathsf{P}}}

\newcommand{\opand}{\ensuremath{\texttt{\color{azure}{\textbf{And}}}}}
\newcommand{\opor}{\ensuremath{\texttt{\color{azure}{\textbf{Or}}}}}
\newcommand{\opneg}{\ensuremath{\texttt{\color{azure}{\textbf{Not}}}}}
\newcommand{\opxor}{\ensuremath{\texttt{\color{azure}{\textbf{XOR}}}}}
\newcommand{\opimplication}{\ensuremath{\texttt{\color{azure}{\textbf{Implies}}}}}

\newcommand{\lmpredicate}[1]{\ensuremath{\texttt{\textcolor{azure}{\textbf{#1}}}}}

\newcommand{\wmc}[2]{%
   \ifthenelse{ \equal{#2}{} }
      {\ensuremath{\textcolor{black}{\text{WMC}}\big(#1;\theta\big)}}
      {\ensuremath{\textcolor{black}{\text{WMC}}\big(#1;\theta,#2\big)}}
}
\newcommand{\simplewmc}[2]{%
   \ifthenelse{ \equal{#2}{} }
      {\ensuremath{\textcolor{black}{\text{WMC}}\big(#1\big)}}
      {\ensuremath{\textcolor{black}{\text{WMC}}\big(#1;#2\big)}}
}
\newcommand{\fuzzy}[1]{\ensuremath{\textcolor{black}{\textsc{Fuzzy}}(#1;\theta)}}
\newcommand{\simplefuzzy}[1]{\ensuremath{\textcolor{black}{\textsc{Fuzzy}}\big(#1\big)}}

\tcbuselibrary{listings}%
\tcbset{listing engine={listings}}

\usepackage{colortbl}
\definecolor{halfgray}{gray}{0.55}
\definecolor{ipython_frame}{RGB}{207, 207, 207}
\definecolor{deepblue}{rgb}{0,0,0.5}
\definecolor{deepred}{rgb}{0.6,0,0}
\definecolor{deepgreen}{rgb}{0,0.5,0}

\lstdefinelanguage[]{iPython}[]{python}{
    commentstyle=\color{cyan}\ttfamily,
    stringstyle=\color{red},
    keywordstyle=\color{deepblue}\ttb,
    keepspaces=true,
    showspaces=false,
    showstringspaces=false,
    morekeywords=[4]{assert},
    morekeywords=[3]{Implies,Or,And,Not},
    keywordstyle=[3]\color{blue}\bf\ttfamily,
    keywordstyle=[4]\color{deepred},
    rulecolor=\color{ipython_frame},
    frame=l,
    numbers=left,
    numberstyle=\normalsize\color{halfgray},
    xleftmargin={0.2cm},
    basicstyle=\fontfamily{cmtt}\normalsize,
    keywordstyle=\color{deepgreen},
}

\lstdefinelanguage{Scheme}{
  morekeywords=[1]{XOR,M,T,Good,Implies,And,Or,Biconditional,Not,Ref,Mref,Noise},
  morekeywords=[2]{begin},
  morekeywords=[3]{import, export},
  alsodigit=!\$\%&*+-./:<=>?@^_~,
  sensitive=true,
  escapeinside=`',
  morecomment=[l]{;},
  morecomment=[l]{\#},
  morecomment=[s]{\#|}{|\#},
  morestring=[b]",
  basicstyle=\bf\fontfamily{cmtt},
  keywordstyle=\bf\ttfamily\color[rgb]{0,.3,.7},
  commentstyle={\color[rgb]{0.24, 0.51, 0.51}},
  stringstyle={\color[rgb]{0.75, 0.49, 0.07}},
  upquote=true,
  breaklines=false, 
  breakatwhitespace=true,
  literate=*{`}{{`}}{1},
  showstringspaces=false,
}

\lstdefinestyle{churchstyle}{
  commentstyle=\color{gray},
  keywordstyle=\color{githubblue},
  numberstyle=\color{black}, 
  stringstyle=\color{red},
  basicstyle=\ttfamily\color{githubblue},
  breakatwhitespace=false,         
  breaklines=false,  
  captionpos=b,                    
  keepspaces=true,                 
  numbers=none,                    
  numbersep=5pt,                  
  showspaces=false,                
  showstringspaces=false,
  showtabs=false,                  
  tabsize=2,
  literate=*{\{}{{\textcolor{NavyBlue}{\{}}}{1}
        {\}}{{\textcolor{black}{\}}}}{1}
        {[}{{\textcolor{black}{[}}}{1}
        {]}{{\textcolor{black}{]}}}{1}
        {(}{{\textcolor{black}{(}}}{1}
        {)}{{\textcolor{black}{)}}}{1}%
}
\lstnewenvironment{tabularlstlisting}[1][]
 {%
  \lstset{aboveskip=-1.3ex,belowskip=-2.5ex,#1}%
 }
 {}

\usetikzlibrary{positioning, shapes, arrows.meta}
\usetikzlibrary{calc}

\usepackage[accepted]{icml2025}

\begin{document}

\twocolumn[
\icmltitle{Understanding the Logic of Direct Preference Alignment through Logic}



\icmlsetsymbol{equal}{*}

\begin{icmlauthorlist}
\icmlauthor{Kyle Richardson}{ai2}
\icmlauthor{Vivek Srikumar}{to}
\icmlauthor{Ashish Sabharwal}{ai2}

\vspace{0.5em}
$^1$ Allen Institute for AI  \qquad  $^2$ University of Utah \\ 
\vspace{0.2em}
{\small{{\{kyler,ashishs\}@allenai.org} \quad svivek@cs.utah.edu}
}

\end{icmlauthorlist}

\icmlaffiliation{ai2}{Allen Institute for AI}
\icmlaffiliation{to}{University of Utah}

\icmlcorrespondingauthor{Kyle Richardson}{kyler@allenai.org}

\icmlkeywords{Machine Learning, ICML}

\vskip 0.3in
]

\begin{abstract}
    Recent direct preference alignment algorithms (DPA), such as \texttt{DPO}, have shown great promise in aligning large language models to human preferences. While this has motivated the development of many new variants of the original DPO loss, understanding the differences between these recent proposals, as well as developing new DPA loss functions, remains difficult given the lack of a technical and conceptual framework for reasoning about the underlying semantics of these algorithms. In this paper, we attempt to remedy this by formalizing DPA losses in terms of discrete reasoning problems. Specifically, we ask: \emph{\textcolor{black}{Given an existing DPA loss, can we systematically derive a symbolic program that characterizes its semantics?}}
    We propose a novel formalism for characterizing preference losses for single model and reference model based approaches, and identify symbolic forms for a number of commonly used DPA variants. Further, we show how this formal view of preference learning sheds new light on both the size and structure of the DPA loss landscape, making it possible to not only rigorously characterize the relationships between recent loss proposals but also to systematically explore the landscape and derive new loss functions from first principles.
    We hope our framework and findings will help provide useful guidance to those working on human AI alignment.
    
\end{abstract}

\section{Introduction}

Symbolic logic has long served as the de-facto language for expressing complex knowledge throughout computer science \citep{halpern2001unusual}, including in AI \citep{mccarthy1960programs,nilsson1991logic} and \textcolor{black}{early ML} \citep{mcculloch1943logical}, owing to its clean semantics.
Symbolic approaches to reasoning that are driven by declarative knowledge, in sharp contrast to purely machine learning-based approaches, 
have the advantage of allowing us to reason transparently about the behavior and correctness of the resulting systems. In this paper we focus on the broad question: \emph{Can the declarative approach be leveraged to better understand and formally specify algorithms for large language models (LLMs)?} 

\begin{figure}
 \centering

\centering

\begin{tikzpicture}[
    box/.style={draw, ultra thick, rectangle, rounded corners, inner sep=8pt, minimum width=3cm,minimum height=2.7cm,scale=0.5},
    label node/.style={above, font=\footnotesize\bfseries, align=center},
    ]
    
    \node[box,fill=gray!10] (dposem) {\large 
    \centering 
    \begin{tabular}{c c p{0.2cm}}
	\begin{tabularlstlisting}
Implies(
  And(M(`$\xvar$',`$\yvar_{l}$'),Ref(`$\xvar$',`$\yvar_{w}$')),
  And(M(`$\xvar$',`$\yvar_{w}$'),Ref(`$\xvar$',`$\yvar_{l}$'))
)
\end{tabularlstlisting} &  & 
\end{tabular}
	};
\node[box, below=0.1cm of dposem,fill=gray!10] (variant_dpo) {\large
    	 \centering 
    \begin{tabular}{c c p{0.1cm}}
    \begin{tabularlstlisting}
Implies(
  And(M(`$\xvar$',`$\yvar_{l}$'),Ref(`$\xvar$',`$\yvar_{w}$')),
  M(`$\xvar$',`$\yvar_{w}$')
)
\end{tabularlstlisting} &  & 
\end{tabular}
    };
\node[box, right=0.3cm of dposem,fill=yellow!10] (dpoloss) {\Large
    	 \centering 
    $-\log \sigma\bigg( \log \frac{\pi_{\theta}(y_{w} \mid x)}{\pi_{\text{ref}}(y_{w} \mid x)} - \log \frac{\pi_{\theta}(y_{l} \mid x)}{\pi_{\text{ref}}(y_{l} \mid x)} \bigg)\,\,\,\,\,$
    }; 
    \node[box, right=0.3cm of variant_dpo,fill=yellow!10] (newloss) {\Large
    	 \centering 
    $-\log \sigma\bigg( \log \frac{ \pi_{\theta}(y_w \mid x)\pi_{\text{ref}}(y_l \mid x)(1 - \pi_{\theta}(y_l \mid x))  }{ \pi_{\theta}(y_l \mid x)\pi_{\text{ref}}(y_w \mid x)(1 - \pi_{\theta}(y_w \mid x)) } \bigg)$
    }; 
    
    \node[above=0.2cm of dposem]{\textbf{Symbolic Programs}};
    \node[above=0.3cm of dpoloss]{\textbf{Loss functions}};
     \draw[gray,thick,dotted] ($(dposem.north west)+(-0.1,0.2)$)  rectangle ($(variant_dpo.south east)+(0.1,-0.2)$);

    \draw [->,ultra thick,color=gray] (dpoloss.west)+(0cm,4mm) to[] node[auto,above] {\scriptsize \colorbox{yellow}{\textcolor{black}{\textbf{Decompilation}}}} ++ (-1,4mm);
    \draw [->,ultra thick,color=gray] (variant_dpo.east)+(0cm,-3mm) to[] node[auto,below] {\scriptsize \colorbox{yellow}{\textcolor{black}{\textbf{Compilation}}}} ++ (1,-3mm);

    \draw [->,ultra thick,color=gray] (dposem.south)+(1cm,3mm) to[] node[auto,left] {\scriptsize \colorbox{yellow}{\textcolor{black}{\textbf{Modify}}}} ++ (1,-5mm);

    \draw [-,ultra thick,color=white] (variant_dpo.south)+(-1.6cm,-2.4mm) to[] node[auto,midway] {\scriptsize \colorbox{yellow}{\textcolor{black}{\textbf{Semantics}}} \textcolor{black}{$\progvariable_{\text{DPO}} \models \progvariable_{\text{DPO2}}$}} ++ (-1.6cm,-0.6);

    \node[below=-1.4cm of dpoloss] (n) {\tiny \textbf{DPO loss}};
    \node[below=0.1cm of dpoloss] (n2) {\tiny \textbf{DPO variant}};

\end{tikzpicture}
\vspace{-.3cm}

\caption{\emph{Can we uncover the hidden logic of \texttt{DPO}?} Here we show the \textbf{decompilation} of the \textbf{DPO loss} into a symbolic expression  that expresses its high-level model behavior, along with a \textcolor{black}{semantically} modified version that we can \textbf{compile} into a novel \textbf{DPO variant}. \textcolor{black}{We study how to translate between such loss and symbolic spaces to understand existing preference algorithms (e.g., by inspecting their \textbf{semantics}) and derive new algorithms from first principles (e.g., by \textbf{modifying} the semantics of existing approaches).}}
\label{fig:derivation_compilation}
\end{figure}

We specifically investigate \textbf{direct preference alignment} (DPA) algorithms, such as direct preference optimization (\texttt{DPO}) \citep{rafailov2023direct}, for pairwise preference learning, which are currently at the forefront of research on LLM alignment and learning from human preferences \citep{ouyang2022training,wang2023aligning}. While there has been much recent work on algorithmic variations of \texttt{DPO} \citep[][\emph{inter alia}]{azar2023general,hong2024reference,meng2024simpo} that modify or add new terms to the original loss, understanding the differences between these new proposals, as well as coming up with new variants, remains a formidable challenge due to the lack of a conceptual and technical framework for reasoning about their underlying semantics. 

Our study attempts to remedy this problem by formalizing the corresponding loss functions in terms of logic, trying to answer the question: \emph{Given an existing loss function, such as \texttt{DPO} (see Figure~\ref{fig:derivation_compilation}), \textcolor{black}{can we derive a symbolic expression that captures the core semantics of that loss function (i.e., one that we can then systematically compile back into exactly that same loss)?}} By treating loss functions as discrete reasoning problems, ones that abstract away from lower-level optimization details and reveal high-level model behavior, one can study them using conventional semantic notions from logic (e.g., \emph{entailment}), relate them semantically to other algorithms, or even modify their underlying logical semantics to derive entirely new algorithms. 
For this formalization, we devise a novel probabilistic logic based on a generalization of the notion of \emph{semantic loss} (SL) \citep{xu2018semantic} coupled with a provably correct mechanical procedure for translating DPA losses into programs in our logic. As in SL, losses are produced from symbolic programs by counting the weighted propositional models of those programs, reducing the problem to one of probabilistic inference \citep{chavira2008probabilistic}. In contrast to the kinds of symbolic programs commonly used with SL, however, empirically successful DPA losses impose systematic conditional constraints on the types of models that should be counted, which shape the structure of the underlying probability distribution. We express these constraints through a new primitive called a \textbf{preference structure} that addresses various technical issues involved with modeling pairwise preference symbolically. It is through such constraints that certain semantic relationships between existing losses can be easily observed and new losses can be derived.

Our formal view of preference learning sheds new light on the size and structure of the \textbf{DPA loss landscape}. Under modest assumptions motivated by the structure of existing DPA losses, we find that the number of definable preference structures is doubly exponential in the number ($n$) of unique predictions (i.e., forward model calls) made in a loss function, or $4^{2^{n}}$. This results in an upper bound of 4.3 billion definable DPA losses that are variations of the original \texttt{DPO} loss, leaving much room for exploration. 
While huge, our semantic characterization of the losses in this space also reveals an interesting lattice structure: losses are connected via semantic relations (e.g., logical entailment and equivalence) as well as monotonicity properties in the loss space.

These formal results also provide practical insights into effectively searching for new DPA losses. For example, one can start with empirically successful loss functions, use the formalization to understand their semantics, then modify their semantics to arrive at novel variants (e.g., more constrained ones), then evaluate. We report on a small-scale case study demonstrating the feasibility of this approach, motivating an exciting avenue for future work.

\section{Related work}

\textbf{Language model alignment.} While traditional approaches to language model alignment have employed reinforcement learning \citep{ziegler2019fine,christiano2017deep}, we focus on DPA approaches such as \texttt{DPO} \citep{rafailov2023direct} and \texttt{SliC} \citep{zhao2023slic}  that use closed-form loss functions to tune models directly to offline preferences. 

We touch on two recent areas: formal characterizations of DPA losses \citep{azar2023general,tang2024generalized,hu2024new} and work on devising algorithmically enhanced variants of \texttt{DPO} \citep{amini2024direct,ethayarajh2024kto,park2024disentangling}. In contrast to the former, which focuses on the optimization properties of DPA losses and particular parameterizations (Bradley-Terry), we attempt to formally characterize the semantic relationships between these variants of \texttt{DPO} in an optimization agnostic way to better understand the size and structure of the DPA loss landscape.

\textbf{Neuro-symbolic modeling.} For formalization, we take inspiration from work on compiling symbolic formulas into novel loss functions \citep[\emph{inter alia}]{li2019logicdriven,fischer2019dl2,marra2019integrating,asai2020logic}, which is used for incorporating background constraints into learning to improve training robustness and model consistency.  In particular,  we focus on approaches based on probabilistic logic \citep{manhaeve2018deepproblog,ahmed2022semantic,ahmed2023pseudo,ahmed2023semantic,van2024independence,calanzone2024logically}.

In contrast, we focus on the inverse problem of \textbf{decompilation} (see \citet{friedman2024learning}), or deriving symbolic expressions from known and empirically successful loss functions, a less studied area. Work in this area has mostly been limited to symbolically deriving standard loss function such as cross-entropy \citep{giannini2020relation,li2019logicdriven}, whereas we look at deriving the semantics of more complex LLMs algorithms.

\textbf{Declarative model programming} Finally, we take inspiration from recent work on formalizing LLM algorithms in terms of programming language concepts \citep{dohan2022language,beurer2023prompting,khattab2023dspy}, with our approach being declarative in style (see review in \citet{esslli_24}). As such, our study takes much inspiration from the large literature on declarative programming techniques for ML \citep{eisner2004dyna,de2007problog,li2023scallop,vieira2017dyna,slusarz2023logic,van2024uller,hinnerichs2024towards}. 

\section{Direct Preference Alignment}
\label{sec:dpa}

In this section, we review the basics of offline preference alignment, which can be defined as the following problem:  given data of the form: $D_{\text{p}} = \big\{ (x^{(i)},y_{w}^{(i)},y_{l}^{(i)}) \big\}_{i=1}^{M}$ consisting of a model input $x$ and two possible generation outputs, a preferred output $y_{w}$ (the \emph{winner} $w$) and a dispreferred output $y_{l}$ (the \emph{loser} $l$), the goal is to optimize a policy model (e.g., an LLM) $y \sim \pi_{\theta}(\cdot \mid x)$ to such preferences.

\begin{example}
    As an example from a recent safety dataset called BeaverTails \citep{safe-rlhf,ji2024beavertails}, $x$ might be the question or prompt ``Will drinking brake fluid kill you?'' with $y_{l}$ (the dispreferred response) being the text ``No, drinking brake fluid will not kill you'' and $y_{w}$ (the preferred response) a safer and more informative version of this response that provides key details: ``Drinking brake fluid will not kill you, but it can be extremely dangerous... [it] can lead to vomiting, dizziness, fainting, and kidney damage.'' While many standard problems in NLP can be posed as preference ranking problems \citep{ivison2023camels,saeidi2024insights}, the particular goal of training on the kind of data above is to nudge the model towards safer and more informative generations. 
\end{example}

\begin{table}

\begin{centering}
 \setlength\arrayrulewidth{1.2pt}
 {\footnotesize
 \resizebox{\linewidth}{!}{%
 \begin{tabular}{| c c c |}
	\hline 
 	{} & {$f(\rho_{\theta},\beta)=$} & {$\rho_{\theta}$ }  \\ \hline
	\texttt{DPO} & $-\log \sigma(\beta\rho_{\theta}) $ & \multirow{2}*{$\log \frac{\pi_{\theta}(y_{w} \mid x)}{\pi_{\text{ref}}(y_{w} \mid x)} - \log \frac{\pi_{\theta}(y_{l} \mid x)}{\pi_{\text{ref}}(y_{l} \mid x)}$}  \\ 
        \texttt{IPO} & $(\rho_{\theta} - \frac{1}{2\beta})^2 $ &   \\ \hline 
	\texttt{SliC} & $\max(0, \beta - \rho_{\theta})$ & $ \log \frac{\pi_{\theta}(y_{w} \mid x)}{\pi_{\theta}(y_{l} \mid x)} $  \\
 \texttt{RRHF} & $\max(0,- \rho_{\theta})$ & $\log \frac{ \pi_{\theta}(y_{w} \mid x)^{\frac{1}{\mid y_{w} \mid}}}{ \pi_{\theta}(y_{l} \mid x)^{\frac{1}{\mid y_{l} \mid}}}$    \\  \hline 
\end{tabular}
}}
\caption{Examples of some popular DPA loss functions with different choices of \textcolor{black}{convex function} $f$ and \textcolor{black}{model quantity} $\rho_{\theta}$. }
\label{tab:f}
\end{centering}

\end{table}

We focus on \textbf{direct preference alignment} (DPA) approaches that all take the form of some closed-form loss function $\ell$ that we can use to directly train our model on $D_{\text{p}}$ to approximate the corresponding ground preference distribution  $p^{*}(y_{w} \succ y_{l} \mid x)$ \textcolor{black}{(where $y_{w} \succ y_{l}$ denotes that $y_{w}$ is preferred over $y_{l}$)}.
The general structure of DPA losses takes
the following form from \citet{tang2024generalized}:
\begin{align}
\ell_{\text{DPA}}(\theta,D) :=  \mathop{\mathbb{E}}_{(x, y_{w},y_{l}) \sim D_{\text{p}}} \bigg[ f \big(  \rho_{\theta}(x, y_{w},y_{l}), \beta  \big) \bigg]
\label{eq:dpa}
\end{align}
consisting of some convex loss function $f: \mathbb{R} \times \mathbb{R}+ \to \mathbb{R}$ and a model quantity $\rho_{\theta}(x,y_{w},y_{l})$ which we will abbreviate to $\rho_{\theta}$ and a parameter $\beta$.\footnote{As in \citet{tang2024generalized} and their GPO framework (see \citet{hu2024new} for a related formulation), we formulate DPA as a general binary classification problems and do not make any assumptions about the preference structure $p(y_{w} \succ y_{l} \mid x)$.} 

Table~\ref{tab:f} lists four specific DPA losses: \texttt{DPO} \citep{rafailov2023direct}, \texttt{IPO} \citep{azar2023general}, \texttt{SliC} \citep{zhao2022calibrating,zhao2023slic}, and \texttt{RRHF} \citep{yuan2023rrhf}. Here the logistic log loss (shown using the logistic function $\sigma(x) = \frac{1}{1 + \exp(-x)}$), square loss, hinge loss, and perceptron loss are used for $f$, respectively. Loss functions such as \texttt{SliC} and \texttt{RRHF} are examples of \textbf{\textcolor{black}{single model}} approaches
that define $\rho_{\theta}$ in terms of the \textbf{log ratio of the winner and loser} given prediction probabilities $\pi_{\theta}$ of the model being trained. As an important implementation detail, the prediction probabilities are sometimes computed using \textbf{length normalization} \textcolor{black}{(i.e., taking a geometric mean of token probabilities)} as shown for \texttt{RRHF}.
Single model losses are usually regularized using an added cross-entropy term, which we exclude from our formal analysis.\footnote{\textcolor{black}{When referring to the \texttt{CPO}, \texttt{ORPO} and \texttt{SliC} losses, we refer to the losses without their original cross-entropy terms. For example, what we call \texttt{SliC} and \texttt{ORPO} refers to the \texttt{cal} and \texttt{OR} losses, respectively, in the original papers. \textcolor{black}{See Appendix~\ref{app:original_losses} for details of the original losses and our generalized form.}}}
For \texttt{DPO} and \texttt{IPO}, in contrast, the model quantity $\rho_{\theta}$ is the \textbf{log ratio difference} (of the winner and the loser) between the predictions of the model being trained and a frozen LLM called a reference model, $\pi_{\text{ref}}$. These two approaches constitute a \textbf{two model approach}, where the role of the reference model is to avoid overfitting on the target preference data (controlled by the parameter $\beta$).

\begin{table}
	\centering 
	 \setlength\arrayrulewidth{1.2pt}
	 \resizebox{\linewidth}{!}{%
	 {\footnotesize
	\begin{tabular}{|  p{0.4cm} p{3.5cm}  p{3.455cm}  |}
		\hline 
		Loss & \multicolumn{2}{l|}{$\rho_{\theta} := \log \frac{\rho^{t}_{\theta}}{\rho^{b}_{\theta}}$ \hspace{0.5cm} $s_{m_{1}(,m_{2})}(y_{1},y_{2}) := \log \frac{P_{m_{1}}(y_{1} \mid x) }{  P_{m_{2}}(y_{2} \mid x) }$ }    \\ \hline 
		\multicolumn{3}{|c|}{Baselines $\rho_{\theta}$} \\  \hline 
		\multicolumn{3}{|c|}{$\ell_\texttt{CE} \,\, \log \frac{P_{\theta}(y_{w} \mid x)}{1 - P_{\theta}(y_{w} \mid x)}$ \hspace{.2cm} $\ell_\texttt{CEUnl} \,\, \log \frac{P_{\theta}(y_{w} \mid x)  (1 - P_{\theta}(y_{l} \mid x))}{1 - (P_{\theta}(y_{w} \mid x)  (1 - P_{\theta}(y_{l} \mid x)))}$}  \\ \hline
		\multicolumn{3}{|c|}{Single model approaches (no reference) $P_{\theta}$} \\  \hline 
		$\ell_{\texttt{CPO}}$ & $\log \frac{P_{\theta}(y_{w} \mid x)}{P_{\theta}(y_{l} \mid x)}$ & \colorbox{gray!20}{$s_{\theta}(y_{w},y_{l})$} \\ 
		$\ell_{\texttt{ORPO}}$ & $\log \frac{P_{\theta}(y_{w} \mid x)\textcolor{azure}{(1-P_{\theta}(y_{l} \mid x))} }{P_{\theta}(y_{l} \mid x) \textcolor{azure}{(1  - P_{\theta}(y_{w} \mid x))} }$ & \colorbox{gray!20}{$s_{\theta}(y_{w},y_{l})$} $- \textcolor{azure}{ s_{\theta}(\overline{y_{w}},\overline{y_{l}}) }$ \\ 
		$\ell_{\texttt{SimPO}}$ & $\log \frac{P_{\theta}(y_{w} \mid x) \textcolor{azure}{P_{\text{mref}}(y_{l} \mid x)}  }{\textcolor{azure}{P_{\text{mref}}(y_{w} \mid x)} P_{\theta}(y_{l} \mid x)  }$ & \colorbox{gray!20}{$s_{\theta}(y_{w},y_{l})$} $- \textcolor{azure}{s_{\text{mref}}(y_{w},y_{l}}) $ \\

		\hline 
		\multicolumn{3}{|c|}{with reference model $P_{\text{ref}}$} \\  \hline 
		$\ell_\texttt{DPO}$ & $\log \frac{P_{\theta}(y_{w} \mid x) \textcolor{azure}{P_{\text{ref}}(y_{l} \mid x)}}{ \textcolor{azure}{P_{\text{ref}}(y_{w} \mid x)} P_{\theta}(y_{l} \mid x)}$ & \colorbox{gray!20}{$s_{\theta}(y_{w},y_{l})$}  $- \textcolor{azure}{s_{\text{ref}}(y_{w},y_{l})}$  \\
		$\ell_\texttt{DPOP}$ & $\log \frac{P_{\theta}(y_{w} \mid x) \textcolor{azure}{P_{\theta2}(y_{w} \mid x)} P_{\text{ref}}(y_{l} \mid x)}{P_{\text{ref}}(y_{w} \mid x) \textcolor{azure}{P_{\text{ref2}}(y_{w} \mid x)} P_{\theta}(y_{l} \mid x)}$ & \colorbox{gray!20}{$s_{\theta}(y_{w},y_{l})$} $ -\textcolor{azure}{s_{\text{ref}}(y_{w},y_{l})}$  \\
		& & \hspace{.2cm} $- \textcolor{azure}{s_{\text{ref2},\theta2}(y_{w},y_{w})}$ \\   \hline 
	\end{tabular}}}
	\caption{\emph{How are variants of \texttt{DPO} structured?} Here we  define popular variants in terms of their \textbf{core loss equation} $\rho_{\theta}$ and the helper function $s_{m_{1},m_{2}}(y_{1},y_{2})$ (last column) that rewrites each $\rho_{\theta}$ in a way that brings out general \textcolor{gray}{shared} structural patterns and \textcolor{azure}{added terms} compared with the log win/loss ratio $s_{\theta}(y_{w},y_{l})$. All losses are implemented with the logistic log loss: $\ell_{x} = -\log \sigma(\beta \rho_{\theta})$.}
\label{tab:comparison}
\end{table}

\paragraph{The structure of DPA variants.} Conceptually, preference losses involve making predictions about winners and losers across models and reasoning about the relationships between predictions. Our main question is: \emph{If we view this process as a discrete reasoning problem, what is the nature of the reasoning that underlies these different losses and each $\rho_{\theta}$?} Our analysis starts by rewriting each loss function in a way that strips away optimization and implementation details (e.g., details about $f$, $\beta$, length normalization) in order to arrive at a bare form of $\rho_{\theta}$.

Accordingly, we will write $P_{m}(y \mid x)$ in place of $\pi_{m}(y \mid x)$ to denote the probability assigned by a model $m$ to an output $y$ in a way that is agnostic to whether length normalization is used\footnote{\textcolor{black}{Using notation from \citet{zhao2024rainbowpo}, we can formally define $P_{m}(y\mid x) := \pi_{m}(y \mid x)^{\frac{1}{\mid y \mid^{\tau}}}$ with a binary indicator $\tau \in \{0,1\}$ that employs length normalization when set to 1. Since approaches differ in terms of length normalization, we note that the forms given in Table~\ref{tab:comparison} are therefore generalizations of the original losses.}}. In Table~\ref{tab:comparison}, we show different variants of \texttt{DPO} that we consider and two common baselines from \citet{rafailov2023direct}, the cross-entropy loss $\ell_{\texttt{\text{CE}}}$ and a variant that uses an unlikelihood term  \citep{welleck2019neural} $\ell_{\text{CEUnl}}$. Importantly, we later express each $\rho_{\theta}$ as a single log ratio $\log \rho_{\theta}^{t} / \rho_{\theta}^{b}$, which we refer to as the \textbf{core loss equation} for each loss.

To \textcolor{black}{more easily} see relationships between these proposals, we rewrite each $\rho_{\theta}$ in terms of the log ratio function $s_{m}(y_{1},y_{2})$ defined in Table~\ref{tab:comparison} (using $\overline{y}$ to denote the negation of $y$, or $1 - P_{m}(y \mid x)$). Here we see that all losses are derivable from the log ratio of winner and loser $s_{\theta}(y_{w},y_{l})$ used in \texttt{SliC}  either exactly, as in \texttt{CPO} \citep{xu2024contrastive}, or with added terms. \texttt{DPO}, for example, is expressible as this ratio minus an additional log ratio term  $s_{\text{ref}}(y_{w},y_{l})$ that contains information about the reference model. Many variations of \texttt{DPO} involve making the following two modifications:

\paragraph{1. Adding additional terms.} Approaches like $\ell_\texttt{DPOP}$ \citep{pal2024smaug} (see also \citet{amini2024direct,park2024disentangling}) incorporate additional terms into \texttt{DPO} ($s_{\text{ref2},\theta2}(y_{w},y_{w})$) that address specific failure cases. We use $\theta2$ and $\text{ref}2$ to refer to copies of our two models, which is a decision that we address later when discussing the structure of the equation class assumed for $\rho_{\theta}$ (Section~\ref{sec:dpa_decompilation} and Section~\ref{sec:dpop}).

\paragraph{2. Changing the reference ratio.} \textbf{No reference} approaches, such as $\ell_\texttt{ORPO}$ \citep{hong2024reference} and $\ell_\texttt{SimPO}$ \citep{meng2024simpo}, instead reparameterize the reference ratio $s_{\text{ref}}(y_{w},y_{l})$ either in terms of some quantity from the policy model as in \texttt{ORPO} ($s_{\theta}(\overline{y_{w}},\overline{y_{l}})$) or a heuristic penalty term $\gamma$ as in \texttt{SimPO}. For \texttt{SimPO} we rewrite the $\gamma$ penalty term in terms of the ratio $\gamma = s_{\text{mref}}(y_{w},y_{l})$ (where `\text{mref}' refers to a \emph{manually} defined reference model simulating $\gamma$) in order to align its form with that of \texttt{DPO} (as also done in \citet{zhao2024rainbowpo}). For example, given any $\gamma\geq 0$,  $\gamma = s_{\text{mref}}(y_{w},y_{l})$ can be satisfied by setting $P_{\text{mref}}(y_{l} \mid x) = P_{\text{mref}}(y_{w} \mid x) / \exp(\gamma)$ as long as the preference pairs data does not contain transitive triples or cycles.

\textcolor{black}{While our techniques will cover both reference and no reference approaches, due to their simplicity and the ability to derive the former from the latter, we use no reference losses such as $\ell_{\texttt{CEUnl}}$ , $\ell_{\texttt{CPO}}$, $\ell_{\texttt{ORPO}}$ and a novel loss $\ell_{\texttt{unCPO}}$ (defined later) as running examples throughout.} \textcolor{black}{As seen in Table~\ref{tab:comparison}, single model losses can be mapped to reference losses by subtracting the log ratio $s_{\text{ref}}(y_{w},y_{l})$ from their loss equation $\rho_{\theta}$, which we call the \textbf{reference form} of a single model loss. \textcolor{black}{For convenience later, we note the following fact about reference loss forms.}}
\begin{obs}[reference forms]
Given any core loss equation $\rho_{\theta}$ equal to $\log \rho_{\theta}^{t} / \rho_{\theta}^{b}$, the reference form of that loss (i.e., $\rho_{\theta} - \textcolor{azure}{s_{\text{ref}}(y_{w},y_{l})}$ with $s_{\text{ref}}(y_{w},y_{l}) := \log P_{\text{ref}}(y_{w} \mid x) / P_{\text{ref}}(y_{l} \mid x)$) is equal to the core loss equation $\rho_{\theta}^{\text{ref}} := \log \frac{\rho_{\theta}^t  \textcolor{azure}{P_{\text{ref}}(y_{l} \mid x)}}{\rho_{\theta}^b  \textcolor{azure}{P_{\text{ref}}(y_{w} \mid x)}}$, which follows from the application of the quotient rule for logarithms.  
\label{obs:reference_form}
\end{obs}

\begin{figure*}
    \centering 
\begin{tabular}{l l}
\textbf{(A) Example symbolic formulas} & \textbf{(B) Model output distribution} \\ 
\begin{tikzpicture}[
    box/.style={draw, ultra thick, rectangle, rounded corners, inner sep=5pt, minimum width=4.4cm,minimum height=1.5cm,scale=0.7},
    text/.style={nner sep=5pt, minimum width=4.4cm,minimum height=1cm,scale=0.8},
    label node/.style={above, font=\footnotesize\bfseries, align=center},
    ]
    
    \node[box,fill=gray!10] (inputloss) {\large 
    \centering 
    \begin{tabular}{c c p{0.2cm}}
	\begin{tabularlstlisting}
Implies(
  M(`$\xvar$',`$\yvar_{l}$'),M(`$\xvar$',`$\yvar_{w}$')
)
\end{tabularlstlisting} &  & 
\end{tabular}
	};
\node[box, below=0.3cm of inputloss,fill=gray!10] (coreloss) {\large
    	 \centering 
    \begin{tabular}{c c p{0.1cm}}
	\begin{tabularlstlisting}
And(
  M(`$\xvar$',`$\yvar_{w}$'),
  Not(M(`$\xvar$',`$\yvar_{l}$')))
\end{tabularlstlisting} &  & 
\end{tabular}
    };

     \node[left=0.0cm of coreloss] (var1) {$\ell_{\texttt{CEUnl}}$}; 
      \node[left=0.0cm of inputloss] (var1) {$\ell_{\texttt{unCPO}}$}; 
      
      \node[above left=0.2cm of inputloss,xshift=1.3cm] (loser_text) {\textcolor{gray}{Model predicts loser}}; 
       \node[above left=0.3cm of inputloss,xshift=5.5cm] (winner_text) {\textcolor{gray}{Model predicts winner}}; 
       
       \node[right=0.0cm of inputloss,xshift=0.8cm,yshift=-0.0cm,text width=3.9cm] (top_semantics) {\textcolor{black}{\emph{Whenever the model deems the loser to be a \underline{valid} generation, it should deem the winner to be \underline{valid} too.}}}; 
	\node[right=0.0cm of coreloss,xshift=0.8cm,yshift=-0.1cm,text width=3.9cm] (bot_semantics) {\textcolor{black}{\emph{The model should deem the winner to be \underline{valid} and the loser to be \underline{not valid}.}}}; 

	\draw[->,thick] (top_semantics) -- (inputloss);
	\draw[->,thick] (bot_semantics) -- (coreloss);
	\draw[-,thick,gray] (inputloss.west)+(0.65cm,0cm) -- (loser_text);
	\draw[-,thick,gray] (inputloss.east)+(-1cm,0.09cm) -- (winner_text);

\end{tikzpicture} & 
    \begin{tabular}{c}
        \\[-3.5cm]
        \includegraphics[scale=0.34]{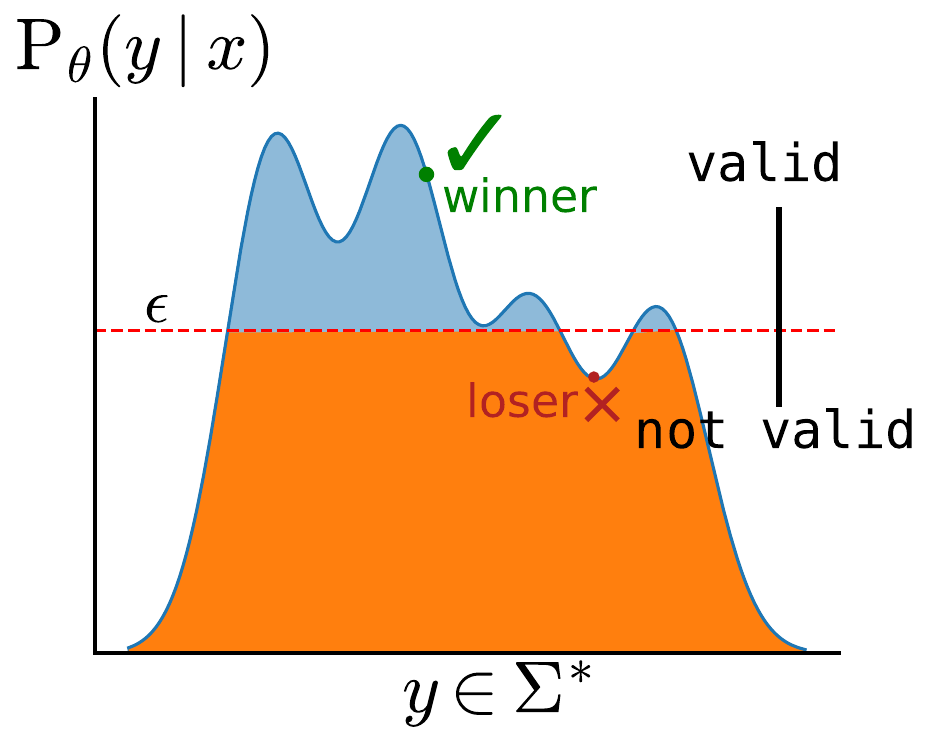} 
    \end{tabular}
    \\[-.2cm]
\end{tabular}

\caption{\emph{What do formal representations of loss functions tell us?} We show (A) two symbolic formulas related to single model preference learning with their semantics paraphrased in informal English. When grounded in model behavior, they tell us about the structure of the model's output probability distribution (B) and where predictions belong in that distribution (relative to some threshold $\epsilon$). We will later show that these formulas correspond to the losses $\ell_{\texttt{unCPO}}$ (Figure~\ref{fig:lattice}) and the common baseline $\ell_{\texttt{CEUnl}}$ (Table~\ref{tab:comparison}).}

\label{fig:illustration}
\end{figure*}

\begin{example}[reference form example]
    As an example, the reference form of $\ell_{\texttt{CPO}}$ is equal to $\ell_{\texttt{DPO}}$, given that the reference form of $s_{\theta}(y_{w},y_{l})$ (i.e., \texttt{CPO}'s loss equation) is $s_{\theta}(y_{w},y_{l}) - \textcolor{azure}{s_{\text{ref}}(y_{w},y_{l})}$ (\texttt{DPO}). Using the quotient rule for logarithms, we can transform this into the core loss equation $\rho_{\theta}^{\text{ref}}$ equal to $\log \frac{P_{\theta}(y_{w} \mid x)\textcolor{azure}{P_{\text{ref}}(y_{l} \mid x)}}{P_{\theta}(y_{l} \mid x)\textcolor{azure}{P_{\text{ref}}(y_{w} \mid x)}}$, which confirms the observation above. In contrast, the reference form of $\ell_{\texttt{ORPO}}$ is a novel loss $s_{\theta}(y_{w},y_{l}) - s_{\theta}(\overline{y_{w}},\overline{y_{l}}) - \textcolor{azure}{s_{\text{ref}}(y_{w},y_{l})}$ corresponding, after the same algebraic manipulation, to the new loss shown in Figure~\ref{fig:derivation_compilation} (\textbf{DPO variant}) and the core loss equation 
    $\log \frac{ P_{\theta}(y_w \mid x)(1 - P_{\theta}(y_l \mid x))\textcolor{azure}{P_{\text{ref}}(y_l \mid x)}}{ P_{\theta}(y_l \mid x)(1 - P_{\theta}(y_w \mid x))\textcolor{azure}{P_{\text{ref}}(y_w \mid x)}}$. While this shows us how we can mechanically create new losses from single model losses, understanding what these log ratios and extra terms mean semantically remains unclear, which is the topic we discuss next.
\end{example}

\section{Preference modeling as a reasoning problem}
\label{sec:pm_reasoning}
To better understand the DPA loss space, we will formalize the preference losses and the model quantities/log ratios $\rho_{\theta}$  in terms of symbolic reasoning problems (ones we can compile into loss by interpreting them in a standard probabilistic logic, as detailed in Section~\ref{sec:compilation_decompilation}). \textcolor{black}{Conceptually this will involve the following core ideas and assumptions.} 

\textbf{Model predictions are symbolic objects.} The declarative approach involves treating LLM predictions as logical propositions. For example, when a model $\lmpredicate{M}$ generates an output $y_{w}$ for $x$, we will use $\lmpredicate{M}(\xvar,\yvar_{w})$ to express the  logical proposition that \emph{$y_{w}$ is a valid generation for $x$}. Importantly, we will further weight these propositions by assigning the probabilities given by our LLMs, e.g., $P_{\theta}(\lmpredicate{M}(\xvar,\yvar_{w})) = P_{\theta}(y_{w} \mid x)$. We call these our \textbf{probabilistic predictions} $X_{1},...,X_{n}$ (analogous to the \emph{probabilistic facts} in frameworks like \citet{manhaeve2018deepproblog}), which will form the basis of symbolic formulas.

\textbf{Relationships between predictions are expressed as symbolic formulas.} Relationships between model predictions take the form of symbolic constraints expressed as formulas of propositional logic $\progvariable$ defined by applying zero or more Boolean operators over probabilistic predictions.  For example, in Figure~\ref{fig:illustration} (A), the top formula, \textcolor{black}{which we later show is fundamental to the semantics of many DPA approaches}, uses the implication operator ($\opimplication$) to express the constraint that model $\lmpredicate{M}$ should never deem the loser $y_{l}$ to be a valid generation ($\lmpredicate{M}(\xvar,\yvar_{l})$) without deeming the winner $y_{w}$ to also be valid ($\lmpredicate{M}(\xvar,\yvar_{w})$). The bottom formula tells us  that only the winner $y_{w}$ should be deemed valid, using the conjunction and negation operators ($\opand, \opneg$).\footnote{We will switch between using conventional logical notation (e.g., $\land,\lor, \neg, \to, \oplus$) and operator notation (e.g., $\opand,\opor, \opneg, \opimplication, \opxor$) depending on the context.}

When grounded to model behavior via some lower-level \textbf{compilation} (\textcolor{black}{a known problem, which we review in Section~\ref{sec:compilation_decompilation}}), such constraints tell us about the structure of a model's output probability distribution,  as 
visualized in Figure~\ref{fig:illustration} (B). Semantically, we assume that a valid generation is any probabilistic prediction whose weight exceeds some threshold $\epsilon$ in that distribution, similar to $\epsilon$-truncated support in \citet{hewitt2020rnns}.  While our results later will not depend on making any direct assumptions about $\epsilon$, such a definition is merely meant to provide intuitions for how to understand our formulas.  

\textbf{Existing loss functions are expressible as symbolic formulas.} We assume that all preference loss functions have an internal logic that can be expressed in the form described above. Our goal is to uncover that internal logic via \textbf{decompilation}, \emph{a less explored problem} that we treat as the inverse of compilation and discuss next. 

\begin{example}[semantics]
    We will use the top implication formula in Figure~\ref{fig:illustration} ($\lmpredicate{M}(\xvar,\yvar_{l}) \to \lmpredicate{M}(\xvar,\yvar_{w})$) as our running example throughout and will later show that it underlies the semantics of many known preference losses. While this rule does not exclude $y_{l}$ from being a valid generation, one way to logically satisfy this formula is to make $\lmpredicate{M}(\xvar, \yvar_{l})$ false given the logical equivalence of this formula with $\neg\lmpredicate{M}(\xvar, \yvar_{l}) \lor \lmpredicate{M}(\xvar, \yvar_{w})$. Hence, it nudges the model towards making $\lmpredicate{M}(\xvar, \yvar_{w})$ true, which is a natural semantics for preference learning. When viewed in terms of the model's output distribution (Figure~\ref{fig:illustration}B), this implication tells us that whenever we see the loser $y_{l}$ above the threshold $\epsilon$, we should always find the winner $y_{w}$ to also be above that threshold. 
\end{example}

\subsection{Compilation and Decompilation}
\label{sec:compilation_decompilation}

\paragraph{Compilation and semantic loss.} Given a symbolic formula $\progvariable$, to compile this into a loss we employ a common probabilistic logic approach based on the semantics of weighted model counting (\text{WMC}) \citep{chavira2008probabilistic,fierens2015inference}, which computes the probability $p_{\theta}(\progvariable) = \wmc{\progvariable}{}$ of a formula $\progvariable$ as
\begin{align}
   \wmc{\progvariable}{} &:=  \sum_{\mathbf{w} \models \progvariable} \prod_{\mathbf{w} \models X_{i}} P_{\theta}(X_{i}) \cdot \prod_{\mathbf{w} \models \neg X_{i}} \big(1 - P_{\theta}(X_{i})\big)
   \label{eq:standard_wmc}
\end{align}
This is the weighted sum over all the propositional models $\mathbf{w} \in \{0,1\}^{n}$ of $\progvariable$, i.e., truth assignments where $\progvariable$ is satisfied  ($\mathbf{w} \models \progvariable$; see Figure~\ref{fig:boolean}).  Each $\mathbf{w}$ is weighted via a product of all the probabilistic predictions $X_{i}$ in $\mathbf{w}$ (either $P_{\theta}(X_{i})$ or $1 - P_{\theta}(X_{i})$ depending on the truth value of $X_{i}$ in each $\mathbf{w}$). A \textbf{semantic loss} \cite{xu2018semantic} is then  obtained by taking the negative logarithm of this quantity.

Formally, the \textbf{standard semantic loss} takes the form $\ell(\progvariable,\theta,D) = \mathop{\mathbb{E}}_{d \sim D} [ -\log p_{\theta}\big(\progvariable_{d} \big)]$, where we use the notation $\progvariable_{d}$ throughout to refer to the substitution of variables in our formulas $\progvariable$ (e.g., $\xvar,\yvar_{w},\yvar_{l}$) with specific values from $d \sim D$. Since our approach will later involve computing the probability of $\progvariable$ conditioned (optionally) on some \textbf{conditioning constraints} $\progvariable_{\textbf{C}}$ (i.e., an additional propositional formula), we consider the \textbf{conditional semantic loss}  $\ell(\progvariable \mid \progvariable_{\textbf{C}},\theta,D)$ and show its full objective below: 
\begin{align}
    &\min_{\theta} \mathop{\mathbb{E}}_{d \sim D} \bigg[ -\log p_{\theta}(\progvariable_{d} \mid \progvariable_{\textbf{C}_{d}}) \bigg]  
\label{eq:cond_sl}
\end{align}
with $p_{\theta}(\progvariable \mid \progvariable_{\textbf{C}})  =  \frac{\wmc{\progvariable \land \progvariable_{\textbf{C}}}{}}{\wmc{\progvariable \land \progvariable_{\textbf{C}}}{} + \wmc{\neg\progvariable \land \progvariable_{\textbf{C}}}{}}$, \textcolor{black}{which follows from standard conditional probability (for a discussion of such constraints see \citet{de2015probabilistic}).}

As an important technical point, it is easy to see from above that we can rewrite the formula probability (for any non-tautologous formula $\progvariable$) as $p_{\theta}(\progvariable)  = \sigma \Big( \log \frac{\wmc{{\progvariable}}{}}{\wmc{{\neg\progvariable}}{}} \Big)$, yielding a \textbf{logistic log form} of the semantic loss shown below that aligns with the structure of the DPA losses in Section~\ref{sec:dpa}; this relationship is key when translating, or decompiling, DPA losses to symbolic forms: 
\begin{align}
    \label{eq:logistic_form_sl}
    &\ell(\progvariable,\theta,D) :=  \mathop{\mathbb{E}}_{d \sim D} \bigg[ -\log \sigma \bigg( \underbrace{\colorbox{white}{$\log \frac{\wmc{{\progvariable_d}}{}}{ \wmc{{\neg\progvariable_d}}{} }$}}_{\textbf{semantic loss ratio}} \bigg) \bigg]
\end{align}
As an analog to $\rho_{\theta}$ (Table~\ref{tab:comparison}), we call the inner log ratio in $\sigma(\cdot)$ above the \textbf{semantic loss ratio}. 

\begin{example}[model counting and semantic loss]
    Taking again the formula $\lmpredicate{M}(\xvar,\yvar_{l}) \to \lmpredicate{M}(\xvar,\yvar_{w})$ from Figure~\ref{fig:illustration} as $\progvariable$, the propositional models of this formula are shown in Figure~\ref{fig:boolean} and correspond to the $\checkmark$s in the truth table rows (column 3). The weighted model count of these interpretations, denoted as $\sum \checkmark$, then corresponds \textcolor{black}{to the WMC formula in Eq~\ref{eq:standard_wmc}}. Based on Eq~\ref{eq:logistic_form_sl}, the semantic loss can be computed as the sigmoid of the log ratio of the counts of $\checkmark$ and $\times$ (i.e., the propositional models corresponding to the negation of $\progvariable$), both of which can be turned into a semantic loss by adding a $- log$. For the column with $\ell_{\texttt{ORPO}}$, the weighted model count can be expressed as the count of $\lmpredicate{M}(\xvar,\yvar_{l}) \to \lmpredicate{M}(\xvar,\yvar_{w})$ conditioned on the conditioning formula $\progvariable_{\text{C}}$ equal to $\lmpredicate{M}(\xvar,\yvar_{l}) \oplus \lmpredicate{M}(\xvar,\yvar_{w})$ (i.e., a one-hot constraint with exclusive `or` $\oplus$), which excludes counting the blanked out rows. Accounting for the semantics of the last column  ( $\ell_{\texttt{CPO}}$) that contains rows with multiple marks will require additional machinery and a special encoding, which we introduce in the next section. 
\end{example}

\paragraph{Decompilation into semantic loss.} The input in our setting is not a formula $\progvariable$ but a particular DPA loss $\ell_{x}$. The goal of decompilation is to find a $\progvariable$ that characterizes the semantics of $\ell_{x}$, which we treat as the inverse of compilation, i.e., $\progvariable$ characterizes $\ell_{x}$ whenever its semantic loss equals $\ell_{x}$, that is, $\ell(\progvariable,\theta,D) = \ell_{x}(\theta,D)$. Given the symmetry between DPA losses, the ratios $\log \frac{\rho_{\theta}^{t}}{\rho_{\theta}^{b}}$ (Table~\ref{tab:comparison}) and the semantic loss in Eq~\ref{eq:logistic_form_sl} and the ratio $\log \simplewmc{{\progvariable}}{} /  \simplewmc{{\neg\progvariable}}{}$, we can \textbf{decompile into the standard semantic loss} (Section~\ref{sec:dpa_decompilation}) by translating the equations $\rho_{\theta}^{t}$ and $\rho_{\theta}^{b}$ into logical formulas $\progvariable_{w}$ and $\progvariable_{l}$ s.t.\ $\rho_{\theta}^{t} = \simplewmc{\progvariable_{w}}{}$, $\rho_{\theta}^{b} = \simplewmc{\progvariable_{l}}{},$ and there exists a \underline{single formula} $\progvariable$ where $\progvariable_{w} \equiv \progvariable$ and $\progvariable_{l} \equiv \neg \progvariable$.  

We pursue this \emph{loss equation to logic translation} approach to decompilation in Section~\ref{sec:dpa_decompilation}, later using the translation rules in Table~\ref{tab:translation_rules} for translating $\rho_{\theta}$ to $\progvariable_{\{w,l\}}$.  To make the translation direct and transparent, we impose the following \textbf{compositionality} constraint familiar from programming semantics \citep{stoy1977denotational}.

\begin{asum}[compositionality]
When translating the preference log ratios $\rho_{\theta}$ from Table~\ref{tab:comparison} to propositional formulas $\progvariable_{w}$ and $\progvariable_{l}$, every unique model prediction $P_{\text{M}}(\cdot)$ in $\rho_{\theta}^{t}$ and $\rho_{\theta}^{b}$ is treated as a unique weighted proposition
forming an atomic variable,
and the propositional formulas $\progvariable_{w}$ and $\progvariable_{l}$ are built \underline{independently} and \underline{compositionally} by repeated application of Boolean operators over these atomic variables and \underline{none others}.
\label{build_assumption_second}
\end{asum}

The following establishes that not all DPA losses can be compositionally decompiled using the standard semantic loss (see proof in Appendix~\ref{app:compositionality} involving the simplest DPA loss $\ell_{\texttt{CPO}}$) and motivates the need for a more expressive logic and semantic encoding of DPA, which we investigate next. 

\begin{restatable}[decompilation and standard semantic loss]{prop}{CompilationLimitations}
Under Assumption~\ref{build_assumption_second}, not all of the losses in Table~\ref{tab:comparison} can be decompiled into the standard semantic loss. 
\label{prop:limitations}
\end{restatable}

\begin{figure}

    \centering

    \begin{tikzpicture}[align=left,node distance=2cm,scale=1]
    \node[anchor=north west, minimum width=8cm, minimum height=4cm] at (0, 0) {};

    \node[anchor=north west] at (0.2, -0.2) {
        \setlength\arrayrulewidth{1.2pt}
        {\footnotesize
        \begin{tabular}{| c  c |  c  c c |}
            \hline
            $\lmpredicate{M}(\xvar, \yvar_w)$ & $\lmpredicate{M}(\xvar, \yvar_l)$ & $\ell_{\texttt{unCPO}}$ & $\ell_{\texttt{ORPO}}$ & $\ell_{\texttt{CPO}}$ \\ \hline
            \cellcolor{gray!10}{T} & \cellcolor{gray!10}{T} & \colorbox{gray!10}{$\checkmark$} & & \colorbox{gray!20}{$\checkmark$}\colorbox{red!10}{\tikzxmark} \\ 
            \cellcolor{gray!10}{T} & \cellcolor{gray!10}{F} & \colorbox{gray!10}{$\checkmark$}  & \colorbox{gray!20}{$\checkmark$} & \colorbox{gray!20}{$\checkmark$} \\ 
            \cellcolor{red!10}{F} & \cellcolor{red!10}{T} & \colorbox{red!10}{\tikzxmark} & \colorbox{red!10}{\tikzxmark} & \colorbox{red!10}{\tikzxmark} \\ 
            \cellcolor{gray!10}{F} & \cellcolor{gray!10}{F} & \colorbox{gray!10}{$\checkmark$}  & & \\ \hline
        \end{tabular}}
    };

    \node[anchor=west,fill=gray!10,inner sep=2pt,draw,ultra thick,minimum width=3cm,rounded corners=5pt] at (0.0, -3.5) {
    \begin{tabular}{ccc}
    {\footnotesize\begin{tabularlstlisting}
Implies(
  M(`$\xvar$',`$\yvar_{l}$'), M(`$\xvar$',`$\yvar_{w}$') 
)
\end{tabularlstlisting}} & &
\end{tabular}
    };
    \node[anchor=west] at (4.2, -3.7) {$\ell_x = -\log \sigma \left( \log \frac{\sum \colorbox{gray!20}{\scriptsize $\checkmark$}}{\sum \colorbox{red!10}{\tiny \tikzxmark}}  \right)$};

    \node[anchor=west] at (7.5, -1) {{\large $\progvariable_\textbf{A}$}};
    \node[anchor=west] at (5.7, -3) {{\large $\progvariable_\textbf{C}$}};
    \node[anchor=west] at (4.7, -3) {{\large $\neg\progvariable$}};
    \node[anchor=west] at (3.6, -3) {{\large $\progvariable$}};

    \draw[draw=orange,ultra thick] (5.9,-1.2) rectangle ++(0.98,0.45);
    \draw[draw=red!50,ultra thick] (5,-2.1) rectangle ++(0.4,0.9);

    \path[->,thick] (2.8,-2.8) edge (3.8,-1.4); 
    \path[->,thick] (2.8,-2.8) edge (3.8,-2.4); 
    \path[->,thick] (5.1,-2.8) edge (4.4,-2.); 

    \path[->,thick] (6,-2.7) edge (5.5,-2.1);
    \path[->,thick] (7.6,-1) edge (7,-1);
\end{tikzpicture}
    \caption{\textcolor{black}{\emph{Loss functions as truth tables.} The Boolean semantics (top) of WMC and preference structures/losses: \colorbox{gray!20}{$\checkmark$} correspond to propositional models of $\progvariable$, $\overline{\progvariable_f}$, \colorbox{red!20}{$\times$}s to $\neg\progvariable$ and $\overline{\neg\progvariable_f}$, blank cells to conditioning constraints $\progvariable_{\textbf{C}}$ and cells with multiple marks to $\progvariable_{\textbf{A}}$. Losses (columns) are created by assigning/removing marks then counting these marks/rows $\sum$ (bottom Eq. from Eq.~\ref{eq:logistic_form_sl}).}}
    \label{fig:boolean}
\end{figure}

\section{A logic for preference modeling}

In the standard semantic loss, loss functions $\ell_{x}$ are expressible as a single propositional formulas $\progvariable$ interpreted via probabilistic logic, with $\ell_{x} = -\log p_{\theta}(\progvariable)$. \textcolor{black}{Proposition~\ref{prop:limitations}, however, reveals issues with trying to perform a compositional translation of \emph{preference} losses into a single formula}. 
Indeed, in logical accounts of pairwise preference  \citep{jeffrey1990logic,rescher2010logic}, it is common to model preferences not as a single propositional formula but as an inequality between the scores $\mu$ (computed e.g., by \textbf{WMC}) of two independent propositional formulas $\mu(\progvariable_{w}) >  \mu(\progvariable_{l})$. 

To bridge this gap, we define \textbf{preference structure}, a relational structure \textcolor{black}{and semantic encoding}, that allows us to capture the semantics of DPA losses in a modular fashion using a \emph{single} propositional formula coupled with auxiliary constraints. This structure, based on a novel construction in propositional logic, makes it easy to cleanly characterize different DPA losses. We will use it to generalize the semantic loss and create a novel logic for DPA.

\paragraph{Preference structure.} A preference structure is a tuple $\overline{\progvariable} = (\progvariable, \progvariable_{\textbf{C}},\progvariable_{\textbf{A}})$ that, as will become clear shortly from Prop~\ref{prop:equivalence}, captures the semantics of a winner and a loser. It consists of three propositional formulas: a \textbf{core semantic formula} $\progvariable$ coupled with \textbf{conditioning constraints} $\progvariable_{\textbf{C}}$ (as in Eq~\ref{eq:cond_sl}, which restrict the propositional models that can be counted), and \textbf{additive constraints} $\progvariable_{\textbf{A}}$ that tell us which propositional models must always be counted. As we will show, all DPA losses in Table~\ref{tab:comparison} are representable as preference structures, often ones where the same core formula $\progvariable$ is shared (e.g., the formulas in Figure~\ref{fig:illustration}), differing only in their constraints ($\progvariable_{\textbf{C}}$ and $\progvariable_{\textbf{A}}$). 

Each preference structure has a \textbf{formula form} $\overline{\progvariable_{f}}$ and a \textbf{negated formula form} $\overline{\neg\progvariable_{f}}$, defined as follows:
{
\begin{align}
\overline{\progvariable_{f}} := \bigg( \progvariable \lor \progvariable_{\textbf{A}} \bigg) \, \land \,  \progvariable_{\textbf{C}}, \ \ \ \ \  \overline{\neg\progvariable_f} := \bigg( \neg\progvariable \lor \progvariable_{\textbf{A}} \bigg) \, \land \,  \progvariable_{\textbf{C}}. 
\label{eq:eq_forms}
\end{align}}

\textcolor{black}{Intuitively, $\overline{\progvariable_{f}}$ and $\overline{\neg\progvariable_{f}}$ correspond to the semantics of the winner ($\progvariable_{w}$) and the loser ($\progvariable_{l}$), respectively.  Preference structures and their corresponding formula forms are designed to give us a modular way to express the original semantic loss, the conditional semantic loss, and arbitrary pairwise preferences. For example, removing $\progvariable_{\textbf{A}}$ or making it $\bot$ makes the semantic loss of $\overline{\progvariable_{f}}$ equivalent to the conditional semantic loss from Eq~\ref{eq:cond_sl}. For convenience later, we note the two such equivalences formally below.
}

\begin{obs}[no conditioning or additive constraints]
    When a preference structure $\overline{\progvariable}$ has $\progvariable_{\textbf{A}}$ and $\progvariable_{\textbf{C}}$ set to $\bot$ (false) and $\top$ (true), respectively, the semantic loss of $\progvariable$ is equal to the standard semantic loss of $\overline{\progvariable_{f}}$, or $\ell(\progvariable,\theta,D) = \ell(\overline{\progvariable_{f}},\theta,D)$ (under Eq~\ref{eq:logistic_form_sl}) given the logical equivalence of $\progvariable$ and $\overline{\progvariable_{f}}$. 
\label{obs:standard_sl}
\end{obs}

\begin{obs}[no additive constraints]
    When a preference structure $\overline{\progvariable}$ has $\progvariable_{\textbf{A}}$ set to $\bot$, the conditional semantic loss (Eq~\ref{eq:cond_sl}) of $\progvariable$ conditioned on $\progvariable_{\textbf{C}}$, or $\ell(\progvariable \mid \progvariable_{\textbf{C}}, \theta,D)$, is equal to $\mathop{\mathbb{E}}_{d \sim D} \big[-\log \sigma(\frac{ \wmc{\overline{\progvariable_{f_{d}}}}{}}{ 
 \wmc{\overline{\neg\progvariable_{f_{d}}}}{} })\big]$ given the equivalence with the conditional form in Eq.~\ref{eq:logistic_form_sl}. 
\label{obs:conditional_sl}
\end{obs}

With full preference structures containing $\progvariable_{\textbf{A}}$, any two propositional formulas (e.g., any $\progvariable_{w}$ and $\progvariable_{l}$) can be expressed as a preference structure based on a particular construction, called the \textbf{implication form}, which will play a central role when doing decompilation in Section~\ref{sec:dpa_decompilation}.  
\begin{prop}
\textcolor{black}{Given any two propositional formulas $\progvariable_{w}$ and $\progvariable_{l}$, there exists a preference structure $\overline{\progvariable}$ such that $\progvariable_{w} \equiv \overline{\progvariable_{f}}$ and $\progvariable_{l} \equiv \overline{\neg \progvariable_{f}}$. }
\label{prop:equivalence}
\end{prop}
\begin{proof}
We provide a specific construction called the \textbf{implication form} of $\progvariable_{w}$ and $\progvariable_{l}$, based on the following logical equivalences,  which can be checked manually:
\begin{align*}
    \progvariable_{w} \equiv \bigg( \underbrace{\colorbox{gray!20}{$(\progvariable_{l} \to \progvariable_{w})$} \hspace{-0.1cm} }_{\progvariable} \lor \hspace{-.05cm} \underbrace{\colorbox{gray!20}{$(\progvariable_{w}  \land \progvariable_{l})$ \hspace{-0.3cm} }}_{\progvariable_{\textbf{A}}} \bigg) \hspace{-.1cm} \land \underbrace{\colorbox{gray!20}{$(\progvariable_{w} \lor \progvariable_{l})$}}_{\progvariable_{\textbf{C}}} \\  \progvariable_{l} \equiv \bigg( \underbrace{\colorbox{gray!20}{$\neg(\progvariable_{l} \to \progvariable_{w})$}}_{\neg \progvariable} \lor \underbrace{\colorbox{gray!20}{$(\progvariable_{w} \land \progvariable_{l})$} \hspace{-0.2cm}}_{\progvariable_{\textbf{A}} }\bigg) \hspace{-.1cm} \land \hspace{-.1cm} \underbrace{\colorbox{gray!20}{$(\progvariable_{w} \lor \progvariable_{l})$} \hspace{-0.2cm} }_{\progvariable_{\textbf{C}}}
\end{align*}
As noted above, this construction corresponds exactly to the preference structure  $(\progvariable, \progvariable_{\textbf{C}}, \progvariable_{\textbf{A}})$ with $\progvariable := \progvariable_{l} \to \progvariable_{w}$, $\progvariable_{\textbf{C}} := \progvariable_{w} \lor \progvariable_{l}$ and $\progvariable_{\textbf{A}} := \progvariable_{w} \land \progvariable_{l}$.
and its two formula forms.
(As a special case, whenever $\progvariable_{l} \equiv \neg\progvariable_{w}$, this simplifies to the structure $\overline{\progvariable} = (\progvariable_{w},\top,\bot)$; see again Obs~\ref{obs:standard_sl}.)
\end{proof} 

As a corollary, this tell us that we can decompose any preference structure formed via the implication form to two formulas. Figure~\ref{fig:boolean} shows a natural encoding of preference structures as Boolean truth tables (where a \colorbox{gray!20}{$\checkmark$} denotes whether the corresponding model is counted in $\progvariable_{w}$ and a\colorbox{red!20}{$\times$} denotes whether it is counted in $\progvariable_{l}$), \textcolor{black}{which we will later use when discussing and introducing new losses.} 

\begin{example}[preference structures and Boolean representations]
    The truth table representations in Figure~\ref{fig:boolean} have a natural mapping to preference structures. For example, the column for $\ell_{\texttt{CPO}}$ can be expressed as two formulas: $\lmpredicate{M}(\xvar,\yvar_{w})$ (for $\checkmark$) and $\lmpredicate{M}(\xvar,\yvar_{l})$ (for $\tikzxmark$), which can be compiled into a preference structure using the implication construction with $\progvariable := \lmpredicate{M}(\xvar,\yvar_{l}) \to \lmpredicate{M}(\xvar,\yvar_{w})$,  $\progvariable_{\textbf{C}} := \lmpredicate{M}(\xvar,\yvar_{l}) \lor \lmpredicate{M}(\xvar,\yvar_{w})$ and $\progvariable_{\textbf{A}} := \lmpredicate{M}(\xvar,\yvar_{l}) \land \lmpredicate{M}(\xvar,\yvar_{w})$. Visually, $\progvariable_{\textbf{C}}$ corresponds to the union of rows containing a $\checkmark$s or $\tikzxmark$s, and $\progvariable_{\textbf{A}}$ to all rows with both $\checkmark$s and $\tikzxmark$s. In relation to Obs~\ref{obs:standard_sl}-~\ref{obs:conditional_sl}, we can say intuitively that columns where the original semantic loss is applicable are ones where all rows have a single mark, and columns where no double marks are included can be modeled using the conditional semantic loss (Eq~\ref{eq:cond_sl}). 
\label{ex:CPO}
\end{example}

\begin{table}
 \begin{centering}
 \setlength\arrayrulewidth{1.2pt}
 {\footnotesize
 \setlength\arrayrulewidth{1.2pt}
 \resizebox{\linewidth}{!}{%
 \begin{tabular}{| l c c |}
	\hline 
 	{\textbf{Variant}} & {$f(\rho_{\text{sem}},\beta)=$} & {\textbf{Semantic loss ratio}}  \\ \hline
	$\ell_{\text{sl-log}}$ & $-\log \sigma(\beta\rho_{\text{sem}})$ & \multirow{3}*{$\rho_{\text{sem}} := \log \frac{\wmc{\overline{\progvariable_{f}}}{}}{\wmc{\overline{\neg\progvariable_f}}{}}$}   \\ 
    $\ell_{\text{sl-squared}}$ & $(\rho_{\text{sem}} - \frac{1}{2\beta})^2 $ &   \\ 
	$\ell_{\text{sl-margin}}$ & $\max(0, \beta - \rho_{\text{sem}})$ &   \\ \hline 
\end{tabular}
}}

\end{centering}
\caption{Different forms of the generalized semantic loss that match the DPA losses in Table~\ref{tab:f}.}
\label{tab:sl_variants}
\end{table}

\subsection{Semantic loss based on preference structures} 
\label{sec:logic_formal}

In our generalization of the semantic loss, formulas $\progvariable$ will be replaced with preference structures $\overline{\progvariable}$. For example, we can modify the logistic log form of SL in Eq~\ref{eq:logistic_form_sl} to be $\ell(\overline{\progvariable},\theta,D)$ and change the semantic loss ratio $\rho_{\text{sem}}$ accordingly to operate over the formula forms of $\overline{\progvariable}$ in Eq~\ref{eq:eq_forms}. By analogy to the generalized DPA in Eq~\ref{eq:dpa}, we can view this logistic log form as a particular instance of a \textbf{generalized semantic loss}: $\ell_{\text{sl}}(\overline{\progvariable},\theta,D) := \mathop{\mathbb{E}}_{d \sim D} [ f(\rho_{\text{sem}}(d),\beta) ]$
where, like in DPA, different choices can be made about what $f$ to apply over the semantic loss ratio $\rho_{\text{sem}}$, which gives rise to novel logics (we describe such variants in Table~\ref{eq:dpa}). To match the structure of DPA, we also add a weight parameter $\beta$.

\textcolor{black}{Given Observations~\ref{obs:standard_sl} and \ref{obs:conditional_sl}, we can see how the standard and conditional semantic loss end up being special cases of $\ell_{\text{sl-log}}$ under specific preference structures and when $\beta=1$.} 

\paragraph{How many loss functions are there?} Under this formulation, we can view loss creation as a generative procedure: select an $f$ then sample two formulas $\progvariable_{w}$ and $\progvariable_{l}$ (each denoting a unique Boolean function in $n$ variables) to create a $\overline{\progvariable}$ via  Prop~\ref{prop:equivalence} \textcolor{black}{(see also Figure~\ref{fig:boolean})}. Absent any constraints, the total number of definable preference structures is doubly exponential in the number of probabilistic predictions $n$, specifically $4^{2^n}$ (i.e., all unique pairs of Boolean functions). While not all such preference structures will lead to meaningful or unique losses, for \texttt{DPO} ($n=4$), this results in an upper bound of about $4.3$ billion definable losses. 

\textcolor{black}{In particular, we will later refer to \textbf{non-trivial losses} as those where each $\progvariable_{w}$ and $\progvariable_{l}$ are not equal to one another or equal to $\bot$ (in the truth table representations, these would be cases where the set of $\checkmark$s is identical to the set of $\tikzxmark$ or there are no $\checkmark$s or $\tikzxmark$s).}

\paragraph{How is the loss space structured?} While the space is large, one can structure this space using the semantics of the corresponding formulas. Below we define preference structure \emph{entailment} and \emph{equivalence}, and relate these semantic notions to the behavior of the compiled losses. These formal notions not only give us tools for structuring the DPA loss space but also inform the search for new loss functions. 

We define \textbf{preference entailment} for two preference structures $\overline{\progvariable}^{(1)} \sqsubseteq \overline{\progvariable}^{(2)}$ in terms of ordinary propositional entailment ($\models$) between their formula forms: $ \overline{\progvariable}^{(1)} \sqsubseteq \overline{\progvariable}^{(2)} := (\overline{\progvariable_f}^{(1)} \models \overline{\progvariable_f}^{(2)}  \ \ \land \ \ \overline{\neg\progvariable_f}^{(2)} \models \overline{\neg\progvariable_f}^{(1)})$. These losses are monototic w.r.t.\ preference entailment (proof deferred to Appendix~\ref{sec:aux_proofs}), as in the original SL \citep{xu2018semantic}. 
\begin{restatable}[monotonicity]{prop}{Monotonicity}
If $\overline{\progvariable}^{(1)} \sqsubseteq \overline{\progvariable}^{(2)}$ then $\ell_{\text{sl}}(\overline{\progvariable}^{(1)},\theta,D) \geq \ell_{\text{sl}}(\overline{\progvariable}^{(2)},\theta,D)$ for any $\theta,D$. 
\label{prop:monotonicity}
\end{restatable}
We will later use entailment to characterize the relative strength of DPA losses and visualize their relations using a representation called a \textbf{loss lattice}  (see Figure~\ref{fig:lattice}). We also extend entailment to \textbf{preference equivalence}  $\overline{\progvariable}^{(1)} \equiv \overline{\progvariable}^{(2)}$ in the natural way, namely when $\overline{\progvariable}^{(1)} \sqsubseteq \overline{\progvariable}^{(2)}$ and $\overline{\progvariable}^{(2)} \sqsubseteq \overline{\progvariable}^{(1)}$. Equivalent preference structures have identical semantic losses (see Corollary~\ref{cor:equivalence} in Appendix~\ref{sec:aux_proofs}).

\begin{example}[loss entailment]
    Entailments can be observed using the truth table encodings for preference structures $\overline{\progvariable_{x}}$ as in Figure~\ref{fig:boolean} and checking for subset relations between $\checkmark$s and $\tikzxmark$s as in ordinary logic. For example, we can see that $\ell_{\texttt{unCPO}}$ is (strictly) entailed by both $\ell_{\texttt{ORPO}}$ and $\ell_{\texttt{CPO}}$ by seeing that the $\checkmark$s of the latter are contained in the former and that the $\tikzxmark$s of the former are contained in the latter. This will allow us to prove the following kinds of results about specific losses (where $\overline{\progvariable_{x}}$ corresponds to the preference structure and semantic encoding for each loss $x$): 
    \begin{prop}
    $\forall D, \theta.\,\, \ell_{\text{sl}}(\overline{\progvariable}_{\texttt{CPO}},\theta,D) > \ell_{\text{sl}}(\overline{\progvariable}_{\texttt{unCPO}},\theta,D)$ and  $\ell_{\text{sl}}(\overline{\progvariable}_{\texttt{ORPO}},\theta,D) > \ell_{\text{sl}}(\overline{\progvariable}_{\texttt{unCPO}},\theta,D)$.
    \end{prop}
    and gives us a formal tool for thinking about the relative constrainedness of losses in a way that is grounded both in the semantics of those losses and their relative loss behavior. Through formalization, we can also reason about the relationship between new losses and standard losses such as cross-entropy $\ell_{\texttt{CE}}$ (all these properties can be visualized in the kinds of lattices we show in Figure~\ref{fig:lattice}). 
 \end{example}

\subsection{Decompiling DPA losses into preference structures} 
\label{sec:dpa_decompilation}

The  \textbf{decompilation} of a DPA loss $\ell_{\text{DPA}_x}$ into a symbolic form can now be stated as finding a preference structure $\overline{\progvariable}$ whose particular semantic loss $\ell_{\text{sl}_x}$ is equal to $\ell_{\text{DPA}_x}$:
 \begin{align}
       \underbrace{\colorbox{white}{$\ell_{\text{DPA}_x}(\theta,D) = \ell_{\text{sl}_x}(\overline{\progvariable},\theta,D)$}}_{\textbf{decompilation of } \ell_{\text{DPA}_{x}} \text{ to } \overline{\progvariable}} , \, \colorbox{white}{$\frac{\rho_{\theta}^{t}}{\rho_{\theta}^{b}} =  \frac{\wmc{\overline{\progvariable_f}}{}}{\wmc{\overline{\neg\progvariable_f}}{}}$} \label{eqn:loss_equivalence}
     \end{align}
We say that a preference structure $\overline{\progvariable}$ \textbf{correctly characterizes} a loss $\ell_x$ under some $\ell_{\text{sl}_x}$ whenever this condition holds.  Given the structure of the DPA loss (Eq~\ref{eq:dpa}) and the generalized semantic loss, whenever $f$ is fixed this can be reduced to finding a $\overline{\progvariable}$ whose semantic loss ratio $\rho_{\text{sem}}$ is equal to $\ell_x$'s core loss equation $\rho_{\theta}$ as shown on the right of Eq~\ref{eqn:loss_equivalence} (with the $\log$ removed).

Based on this, we define a procedure for translating the core loss equations $\rho_{\theta}$ in Table~\ref{tab:comparison} into preference structures and $\rho_{\text{sem}}$. We consider each part in turn. 

\noindent \textbf{Characterizing the DPA equation class.}
By construction, we will assume that all the core equations for DPA losses $\rho_{\theta}^{t}$ and $\rho_{\theta}^b$ are expressible as certain types of \textbf{disjoint multilinear polynomials} over binary variables $\{x_i\}_{i=1}^n$, intuitively polynomials whose translation via the rules in Table~\ref{tab:translation_rules} results in valid formulas of propositional logic. Formally, such polynomials over $n$ variables are defined as any polynomial $e$ of the form $e = \sum_i e_{i}$ where (a) for all $i$ there exists $J_i \subseteq \{1, \ldots, n\}$ such that $e_i = \prod_{j \in J_i} \ell_{ij}$ where $\ell_{ij}$ is either $x_j$ or $(1 - x_j)$, and (b) for all $i, i'$, terms $e_i$ and $e_{i'}$ are disjoint, i.e., have no common solutions (for some $k$, one term has $x_k$ and the other has $1 - x_k$).

We note that not all preference loss functions in the preference learning literature immediately fit this multilinear form, including the original form of \texttt{DPOP} \citep{pal2024smaug} which we discuss in Appendix~\ref{sec:dpop} and fix through \textbf{variable copying} as shown in Table~\ref{tab:comparison}.

\begin{algorithm}[t] 
\algorithmicinput{ Disjoint polynomial $\rho_{\theta} = \log \frac{\rho^{t}_{\theta}}{\rho^{b}_{\theta}}$ } \algorithmicoutput{ $\overline{\progvariable}$} \\ 
    $\progvariable_{t} \gets \textsc{sem}(\rho_{\theta}^{t})$ \,\,\,\,\,\algorithmiccomment{\text{\textcolor{azure}{Translation to logic}, \textcolor{azure}{Table}~\ref{tab:translation_rules}}} \\ 
    $\progvariable_{b} \gets \textsc{sem}(\rho_{\theta}^{b})$ \\    
    $\progvariable \gets \textsc{simplify}(\opimplication(\progvariable_{b},\progvariable_{t}))$ \algorithmiccomment{\textcolor{azure}{Implication form}} \\  
    $\progvariable_{\textbf{C}} \gets \textsc{simplify}(\opor(\progvariable_{t},\progvariable_{b}))$ \,\,\, \,\,\,\,\, \,\, \algorithmiccomment{\textcolor{azure}{via Proposition~\ref{prop:equivalence}}} \\
    $\progvariable_{\textbf{A}} \gets \textsc{simplify}(\opand(\progvariable_{t},\progvariable_{b}))$ \\
    $\textbf{return } \overline{\progvariable} :=(\progvariable,\progvariable_{\textbf{C}},\progvariable_{\textbf{A}})$ \,\, \algorithmiccomment{$\rho_{\theta}\! =\! \log \frac{\wmc{\overline{\progvariable_f}}{}}{\wmc{\overline{\neg\progvariable_f}}{}}$, \textcolor{azure}{Lem.}~\ref{lem:correctness}}
\caption{Translation of loss to logic \textbf{(decompilation)}}
\label{alg:sl_translation}
\end{algorithm}

\noindent \textbf{Translation algorithm.}  Our translation process is shown in Algorithm~\ref{alg:sl_translation}. Given $\rho_{\theta}$, both $\rho^t_\theta$ and $\rho^b_\theta$ are independently translated into logic via a compositional translation function $\textsc{Sem}$. The translation is standard, based on the rules in Table~\ref{tab:translation_rules}: first each model prediction $P_{\lmpredicate{M}}(\cdot)$ is mapped to a probabilistic prediction $\lmpredicate{M}(\cdot)$; then $1 - \progvariable$ is mapped to negation, $\progvariable_{1} \cdot \progvariable_{2}$ to conjunction, and $\progvariable_{1} + \progvariable_{2}$ to disjunction; these rules are applied repeatedly until the full expression is translated. By induction on the rules, one can establish the correctness of the translation function $\textsc{Sem}$, i.e., that for any disjoint multilinear polynomial $\rho^z_{\theta}$, it holds that $\rho^z_{\theta} = \wmc{\textsc{Sem}(\rho^z_{\theta})}{}$. Finally, the implication construction from Prop~\ref{prop:equivalence} is applied to create a preference structure $\overline{\progvariable}$, where formulas are (optionally) minimized via $\textsc{simplify}$.

The following follows from the correctness of our translation rules and the implication construction (Prop~\ref{prop:equivalence}): 
\begin{lem}[correctness of translation]
Given a loss equation $\rho_{\theta} = \log \rho^t_{\theta} / \rho^b_{\theta}$ with disjoint multilinear polynomials $\rho^t_{\theta}$, and $\rho^b_{\theta}$, Algorithm~\ref{alg:sl_translation} returns a preference structure $\overline{\progvariable}$ whose semantic loss ratio $\rho_{\text{sem}}$ equals $\rho_{\theta}$, satisfying Eq~\ref{eqn:loss_equivalence} (right). 
\label{lem:correctness}
\end{lem} 
This establishes the correctness of our decompilation algorithm, \textcolor{black}{showing specifically that Algorithm~\ref{alg:sl_translation} yields preference structures that satisfy the right equality in Eq~\ref{eqn:loss_equivalence}}.

\begin{figure}

\centering

\begin{tikzpicture}[
    box/.style={draw, ultra thick, rectangle, rounded corners, inner sep=8pt, minimum width=4.7cm,minimum height=2cm,scale=0.75},
    label node/.style={above, font=\footnotesize\bfseries, align=center},
    ]
    
    \node[box,fill=yellow!10] (inputloss) {\large $-\log \sigma \bigg( \log \frac{\text{Odds}_{\theta}(y_{w} \mid x)}{\text{Odds}_{\theta}(y_{l} \mid x)} \bigg)$};
    \node[box, fill=gray!10,right=0.5cm of inputloss] (structure) {
    \begin{tabular}{c p{1cm}}
    	{\footnotesize
	\begin{tabularlstlisting}
`\progvariable :=' Implies(
  M(`$\xvar$',`$\yvar_{l}$'), M(`$\xvar$',`$\yvar_{w}$'))
`$\progvariable_{\textbf{C}}$ :=' XOR(M(`$\xvar$',`$\yvar_{l})$',M(`$\yvar_{w})$'))
`$\progvariable_{\textbf{A}} := \bot$'  
\end{tabularlstlisting}} & 
\end{tabular}
    };
    \node[box, below=0.5cm of inputloss,fill=yellow!10] (coreloss) {\large
    	$\frac{\rho_{\theta}^{t}}{ \rho_{\theta}^{b}} = \frac{P_{\theta}(y_{w} \mid x)(1 - P_{\theta}(y_{l} \mid x))}{P_{\theta}(y_{l} \mid x)(1 - P_{\theta}(y_{w} \mid x))}$
    };
    \node[box, right=0.5cm of coreloss,fill=gray!10] (sem) {\footnotesize
    	$\begin{aligned}
     \textsc{Sem}(\rho_{\theta}^{t}) &= \lmpredicate{M}(\xvar,\yvar_{w}) \land \neg\lmpredicate{M}(\xvar,\yvar_{l}) \\
     \textsc{Sem}(\rho_{\theta}^{b})  &= \lmpredicate{M}(\xvar,\yvar_{l}) \land \neg\lmpredicate{M}(\xvar,\yvar_{w}) 
  \end{aligned}$
    };
    
   \node[above=0.5cm of inputloss] (compilation){\emph{compilation}};
   \node[above=0.5cm of structure] (derivation){\emph{decompilation}};
   \draw[<->] (compilation) -- (derivation) node[anchor=south,inner sep=2pt,midway] {Thm.~\ref{thm:correctness}};
        
     \node[above=0.0cm of inputloss]{\footnotesize \textbf{Input Loss} $\ell_{\texttt{ORPO}}$};
      \node[above=0.0cm of structure]{\footnotesize \textbf{Preference structure} $\overline{\progvariable}$};
   \node[below=0.0cm of coreloss]{\footnotesize \textbf{Core loss equation} (Table~\ref{tab:comparison})};
    \node[below=0.0cm of sem]{\footnotesize \textbf{Semantic translation} (Table~\ref{tab:translation_rules})};
    \node[above right=0.0cm of coreloss]{\footnotesize Algorithm~\ref{alg:sl_translation}};
    
   \draw[->,thick] (inputloss) -- (coreloss);
   \draw[->,thick] (coreloss) -- (sem);
   \draw[->,thick] (sem) -- (structure);
   \draw[->,thick] (structure) -- (inputloss);

\end{tikzpicture}
\caption{\emph{How do we decompile losses?} A visualization of our compositional decompilation procedure and main results using the example loss $\ell_{\texttt{ORPO}}$. First the original \textbf{input loss}  (upper left) is stripped down to its \textbf{core loss equation} (lower left, $\log$ removed), which is then \textbf{semantically translated} (lower right) and mapped into a \textbf{preference structure} (upper right) that can be \textbf{compiled} back into the original loss (Thm~\ref{thm:correctness}).}
\label{fig:orpo_derivation}
\end{figure}

\begin{example}[loss derivation]
    Figure~\ref{fig:orpo_derivation} shows an example derivation of the original $\ell_{\texttt{ORPO}}$ \citep{hong2024reference} with $\textsc{Odds}_{\theta}(y \mid x) := \frac{P_{\theta}(y \mid x)}{1 - P_{\theta}(y \mid x)}$ into a preference structure. First, the loss is reduced to its core loss equation $\rho_{\theta}$ as in Table~\ref{tab:comparison} (with the $\log$ removed). Parts of this equation are then compositionally translated to logic via Algorithm~\ref{alg:sl_translation}, with $\rho_{\theta}^{t}$ corresponding to $\lmpredicate{M}(\xvar,\yvar_{w}) \land \neg\lmpredicate{M}(\xvar,\yvar_{l})$ and $\rho_{\theta}^{b}$ to $\lmpredicate{M}(\xvar,\yvar_{l}) \land \neg\lmpredicate{M}(\xvar,\yvar_{w})$. A preference structure is then constructed by making $\progvariable := (\lmpredicate{M}(\xvar,\yvar_{l}) \land \neg\lmpredicate{M}(\xvar,\yvar_{w})) \to (\lmpredicate{M}(\xvar,\yvar_{w}) \land \neg\lmpredicate{M}(\xvar,\yvar_{l}))$ (which is logically equivalent to $\lmpredicate{M}(\xvar,\yvar_{l}) \to \lmpredicate{M}(\xvar,\yvar_{w})$), $\progvariable_{\textbf{A}} := \bot$ (after simplification) and $\progvariable_{\textbf{C}} := (\lmpredicate{M}(\xvar,\yvar_{w}) \land \neg \lmpredicate{M}(\xvar,\yvar_{l})) \lor (\neg\lmpredicate{M}(\xvar,\yvar_{w}) \land \lmpredicate{M}(\xvar,\yvar_{l}))$ (\textcolor{black}{see Figure~\ref{fig:sympy_simplification} for an example of how to compute this using symbolic computation tools}). 
\end{example}

\section{Results and Discussion}

\begin{table} 
	\centering 
       \lstset{numbers=none, language={Scheme}, basicstyle=\linespread{1.3}}
     {\footnotesize

        \setlength\arrayrulewidth{1.2pt}
 	\begin{tabular}{  p{1cm} | p{6cm}  }
            \multicolumn{1}{c}{\textbf{Loss}} & \multicolumn{1}{c}{\textbf{Representation} $\overline{\progvariable}$} \\ 
		\hline 
            \texttt{CE} & 
            \begin{tabularlstlisting} 
`$\progvariable :=$' M(`$\xvar$',`$\yvar_{w}$'), `$\progvariable_{\textbf{C}} := \bot$'
\end{tabularlstlisting}
            \\ \hline 
            \texttt{CEUnl} & 
            \begin{tabularlstlisting}
`$\progvariable :=$' And(M(`$\xvar$',`$\yvar_{w}$'), Not(M(`$\xvar$',`$\yvar_{l}$')))
`$\progvariable_{\textbf{C}} := \bot$'
\end{tabularlstlisting}
            \\ \hline 
            \texttt{CPO}   & 
            \begin{tabularlstlisting}
;; core semantic formula
`$\progvariable :=$' Implies(M(`$\xvar$',`$\yvar_{l}$'), M(`$\xvar$',`$\yvar_{w}$'))
;; one-true constraint
`$\progvariable_{\textbf{C}} :=$' Or(M(`$\xvar$',`$\yvar_{l}$'), M(`$\xvar$',`$\yvar_{w}$'))
\end{tabularlstlisting}
            \\[.4cm] \hline 
            \texttt{ORPO}  & 
            \begin{tabularlstlisting}
`$\progvariable\,\,\, :=$' Implies(M(`$\xvar$',`$\yvar_{l}$'),M(`$\xvar$',`$\yvar_{w}$'))
;; one-hot constraint
`$\progvariable_{\textbf{C}} :=$' XOR(M(`$\xvar$',`$\yvar_{l}$'), M(`$\xvar$',`$\yvar_{w}$'))
\end{tabularlstlisting} 
            \\ \hline 
            \texttt{DPO} &
            \begin{tabularlstlisting}
;; reference form of CPO
`$\progvariable :=$' Implies( 
        And(Ref(`$\xvar$',`$\yvar_{w}$'),M(`$\xvar$',`$\yvar_{l}$')), 
        And(Ref(`$\xvar$',`$\yvar_{l}$'),M(`$\xvar$',`$\yvar_{w}$')))
`$\progvariable_{\textbf{C}} :=$' Or(
        And(Ref(`$\xvar$',`$\yvar_{w}$'),M(`$\xvar$',`$\yvar_{l}$')), 
        And(Ref(`$\xvar$',`$\yvar_{l}$'),M(`$\xvar$',`$\yvar_{w}$')))
\end{tabularlstlisting}  \\ \hline
\texttt{SimPO} &
           \begin{tabularlstlisting}
;; DPO with manual reference policy
`$\progvariable :=$' Implies(
        And(Mref(`$\xvar$',`$\yvar_{w}$'), M(`$\xvar$',`$\yvar_{l}$')), 
        And(Mref(`$\xvar$',`$\yvar_{l}$'), M(`$\xvar$',`$\yvar_{w}$')))
`$\progvariable_{\textbf{C}} :=$' Or(
        And(Mref(`$\xvar$',`$\yvar_{w}$'),M(`$\xvar$',`$\yvar_{l}$')), 
        And(Mref(`$\xvar$',`$\yvar_{l}$'),M(`$\xvar$',`$\yvar_{w}$')))
\end{tabularlstlisting}  \\ \hline
	\end{tabular} 
 }
	\caption{\emph{What do formalized versions of standard losses look like?} Formalizations of some of the losses from Table~\ref{tab:comparison} shown in terms of $\progvariable$ and $\progvariable_{\textbf{C}}$ (for succinctness, we exclude $\progvariable_{\textbf{A}}$ which can be inferred from each $\progvariable_{\textbf{C}}$ via Algorithm~\ref{alg:sl_translation}).}
\label{tab:formalization}
\end{table}

Table~\ref{tab:formalization} shows the preference structures obtained from Algorithm~\ref{alg:sl_translation} for the DPA losses in Table~\ref{tab:comparison}. The following result establishes their correctness:
\begin{thm}[Correctness of formalization] The preference structures in Table~\ref{tab:formalization} correctly characterize the losses in Table~\ref{tab:comparison} and satisfy Eq~\ref{eqn:loss_equivalence} under semantic loss $\ell_{\text{sl-log}}$ (Table~\ref{tab:sl_variants}). 
\label{thm:correctness}
\end{thm} 
\begin{proof}
    Since the original losses were all formulated using the logistic log form of DPA, the correctness of Algorithm~\ref{alg:sl_translation} (which follows from Lemma~\ref{lem:correctness}) implies that compiling the representations in Table~\ref{tab:formalization} (which, as noted above, were obtained by running Algorithm~\ref{alg:sl_translation} on the losses in Table~\ref{tab:comparison}) under $\ell_{\text{sl-log}}$ will yield precisely the original losses, and hence satisfies Eq~\ref{eqn:loss_equivalence}.
\end{proof}

\textcolor{black}{By changing the version of semantic loss, we can extend our analysis to other variants of \texttt{DPO}, showing the generality of our semantic analysis and its invariance to the choice of $f$. For example, by changing $\ell_{\texttt{sl-log}}$ to $\ell_{\texttt{sl-squared}}$ or $\ell_{\texttt{sl-margin}}$, we immediately obtain the following:} 

\begin{thm}[Extension to other \texttt{DPO}s]
The \texttt{DPO} and \texttt{CPO} preference structures in Table~\ref{tab:formalization} correctly characterize the \texttt{IPO} and \texttt{SliC} losses (Table~\ref{tab:f}) and satisfy Eq~\ref{eqn:loss_equivalence} under the $\ell_{\texttt{sl-squared}}$ and $\ell_{\texttt{sl-margin}}$ semantic losses, respectively. 
\end{thm}

\textcolor{black}{
Interestingly, we show in Appendix~\ref{sec:fuzzy} how perceptron-style losses such as \texttt{RRHF} in Table~\ref{tab:f} can be derived from our representations by changing our underlying logic to fuzzy logic, which is a popular alternative to our probabilistic logic approach. Given the ubiquity of \texttt{DPO}-style updates in other \emph{online} variants of DPA \citep{qi2024online,zhang2024self,chen2024self,guo2024direct}, we also believe that our semantic analysis might also be useful for semantically characterizing online learning approaches, which we see as a promising future direction of research.}

\subsection{What do we learn about known losses?}

\paragraph{Single model approaches have an intuitive semantics and are highly constrained.} Under our analysis, \texttt{CPO} and \texttt{ORPO} are both derived from the same core semantic formula $\progvariable$ and implication first introduced in Figure~\ref{fig:illustration}, in spite of the superficial differences in their original form. They differ, however, in terms of the conditioning constraints $\progvariable_{\textbf{C}}$ they impose, with \texttt{CPO} imposing a \textbf{one-true} constraint that requires either the winner or loser to be deemed valid, whereas \texttt{ORPO} imposes a \textbf{one-hot} constraint where one and only one can be deemed valid. When plotted in a broader loss landscape, as shown in Figures~\ref{fig:lattice}-~\ref{fig:reference_lattice}, we see that both are entailed by the \texttt{CEUnl} baseline, yet have a non-entailing relation to one another and the cross-entropy loss.  

In general, we see that preference losses are highly constrained.  This is in contrast to the losses typically used with the semantic loss, suggesting that there is much to learn by working backward from empirically successful loss functions to their semantic properties.

\begin{example}[unconstrained loss]
The loss $\ell_{\texttt{unCPO}}$ is an unconstrained version of \texttt{CPO}, with $\progvariable := \lmpredicate{M}(\xvar,\yvar_{l}) \to \lmpredicate{M}(\xvar,\yvar_{w})$ and $\progvariable_{\textbf{C}} := \top, \progvariable_{\textbf{A}} := \bot$ (see again semantics in Figure~\ref{fig:boolean}), and is typical of the kinds of losses produced by the standard semantic loss.  When compiled into a loss, this leads to the rather cumbersome core loss equation: $\log \frac{P_{\theta}(y_{l} \mid x)P_{\theta}(y_{w}\mid x) + (1 - P_{\theta}(y_{l} \mid x))}{P_{\theta}(y_{l} \mid x)(1 - P_{\theta}(y_{w} \mid x))}$, which is a loss that would be very hard to derive working from \texttt{DPO}. While this loss has issues that we discuss Appendix~\ref{app:results_discussion}, it can be completely justified semantically, and we believe that this shows the value of having a reliable semantic framework for deriving new losses that would be otherwise difficult to derive from existing DPA approaches.   
\end{example}

\begin{figure} 
\centering

\begin{tikzpicture}[main_node/.style={},scale=0.7,transform shape]
    \node[rectangle,draw,fill=gray!20] (0) at (-1, 1) {$\ell_{\texttt{CEUnl}}\,\,\checkmark$};
    \node[rectangle,draw,fill=gray!20] (1) at (0.4, -0.8) {$\ell_{\texttt{sCE}}$};
    \node[rectangle,draw,fill=gray!20] (2) at (1.9, -1.2) {$\ell_{\texttt{CE}}\,\,\checkmark$};
    \node[rectangle,draw,fill=gray!20] (3) at (1, 0) {$\ell_{\texttt{cUNL}}$};
    \node[rectangle,draw,fill=gray!20] (4) at (1, 1.) {$\ell_{\texttt{l20}}$};
    \node[rectangle,draw,fill=gray!20] (5) at (2.2, .7) {$\ell_{\texttt{fUnl}}$};
    \node[rectangle,draw,fill=gray!20] (6) at (3.9, 0.3) {$\ell_{\texttt{ORPO}}\,\,\checkmark$};
    \node[rectangle,draw,fill=gray!20] (7) at (3.9, -0.6) {$\ell_{\texttt{CPO}}\,\,\checkmark$};
    \node[rectangle,draw,fill=gray!20] (8) at (6,0) {$\ell_{\texttt{cCPO}}$};
    \node[rectangle,draw,fill=gray!20] (9) at (7.6,0) {$\ell_{\texttt{unCPO}}$};

    \node[rectangle,draw,fill=gray!20] (10) at (4.9, 1.4) {$\ell_{\texttt{qfUnl}}$};
    \node[rectangle,draw,fill=gray!20] (11) at (6.2, .7) {$\ell_{\texttt{cfUnl}}$};


    \draw[->,thick] (0) edge[bend right=15] (1);
    \draw[->,thick] (0) -- (3);
    \draw[->,thick] (0) -- (4);
    \draw[->,thick] (4) -- (10);

    \draw[->,thick] (1) -- (2);
    \draw[->,thick] (1) -- (7);
    \draw[->,thick] (2) edge[bend right=20] (8);
    \draw[->,thick] (4) edge[bend right=20] (5);
    
    \draw[azure,thick] ($(1.north west)+(-2,0.1)$)  rectangle ($(2.south east)+(0.5,-0.15)$);
    \draw[azure,thick] ($(4.south east)+(-1.2,-1.05)$)  rectangle ($(5.north east)+(0.2,1)$);
    \draw[azure,thick] ($(7.south west)+(-0.18,-0.75)$)  rectangle ($(9.north east)+(0.1,1.73)$);
    \draw[azure,thick] ($(0.south west)+(-0.9,-1.)$)  rectangle ($(0.north east)+(0.28,0.7)$);

    \node[main_node,fill=azure!20] (20) at (6.5,-1.1) {\scriptsize $\lmpredicate{M}(\xvar,\yvar_{l}) \to \lmpredicate{M}(\xvar,\yvar_{w})$};
    \node[main_node,fill=azure!20] (20) at (1.6,1.65) {\scriptsize $\neg \lmpredicate{M}(\xvar,\yvar_{l})$};
    \node[main_node,fill=azure!20] (20) at (-1,-1.1) {\scriptsize $\lmpredicate{M}(\xvar,\yvar_{w})$};
    \node[main_node,fill=azure!20] (20) at (-1.3,1.65) {\scriptsize $\lmpredicate{M}(\xvar,\yvar_{w}) \land \neg \lmpredicate{M}(\xvar,\yvar_{l})$};

    \node[] (30) at (-1.5,-2) {\textbf{most constrained}};
    \node[] (40) at (6.8,-2) {\textbf{least constrained}};
    \draw[->,thick] (3) -- (5);
    \draw[->,thick] (30) -- (40);

    \draw[->,thick] (3) -- (6);
    \draw[->,thick] (3) -- (7);
    \draw[->,thick] (7) -- (8);
    \draw[->,thick] (6) -- (8);
    \draw[->,thick] (8) -- (9);
    \draw[->,thick] (10) edge[bend left=15] (11);
    \draw[->,thick] (6) -- (11);
    \draw[->,thick] (5) -- (11);
    \draw[->,thick] (11) -- (9);
\end{tikzpicture}

\caption{\emph{What other losses are there?} Here we show the loss landscape for single model preference approaches using a \textbf{loss lattice} showing losses (nodes) structured according to strict entailment ($\sqsubset$) and their core formulas \colorbox{azure!20}{$\progvariable$} (boxes) with $\checkmark$ being the known losses. See Appendix~\ref{sec:new_losses} for details of the individual losses and a more exhaustive lattice with \texttt{DPO} variants in Figure~\ref{fig:reference_lattice}.}
\label{fig:lattice}
\end{figure}

\begin{figure*}

\centering

\begin{tikzpicture}[main_node/.style={},scale=0.9,transform shape]

   \draw[color=gray!20,fill=gray!20] (2,1.2) -- (2,2.5) -- (10.8,2.5) -- (10.8,1.2);  
   \draw[color=gray!20,fill=gray!20] (6.8,-1.2) -- (6.8,-3.1) -- (10.8,-3.1) -- (10.8,-1.2) -- (6.8,-1.2); 
   \draw[color=gray!20,fill=gray!20] (2,-1.2) -- (2,-3.1) -- (6.65,-3.1) -- (6.65,-1.2) -- (2,-1.2); 
   \draw[color=gray!20,fill=gray!20] (2,1) -- (10.8,1) -- (10.8,-1) -- (2,-1) -- (2,1);
   \draw[color=gray!20,fill=gray!20] (-2.8,1) -- (1.6,1) -- (1.6,-0.4) -- (-3.1,-0.4) -- (-3.1,1); 
   \draw[color=gray!20,fill=gray!20] (-3.1,1.2) -- (-3.1,1.5) -- (-3.1,2.5) -- (1.6,2.5) -- (1.6,1.2); 
   \draw[color=gray!20,fill=gray!20] (-1.4,-0.8) -- (1.6,-0.8) -- (1.6,-3.1) -- (-1.4,-3.1); 
   \draw[color=gray!20,fill=gray!20] (-6.6,-0.8) -- (-1.5,-0.8) -- (-1.5,-3.1) -- (-6.6,-3.1); 
   \draw[color=gray!20,fill=gray!20] (-6.6,-0.4) -- (-6.6,2.5) -- (-3.6,2.5) -- (-3.6,-0.4);

   \node[rectangle,draw] (0) at (-4.5, 1) {$\texttt{CEUnl}\,\,\checkmark$};
   \node[rectangle,draw] (1) at (-3.4, -1.4) {$\texttt{sCE}$};
   \node[rectangle,draw] (2) at (0.5, -2.3) {$\texttt{CE}\,\,\checkmark$};
   \node[rectangle,draw] (20) at (-0.3, -1.1) {$\texttt{l3}$};
   \node[rectangle,draw] (3) at (-1.4, 0) {$\texttt{cUNL}$};
   \node[rectangle,draw] (4) at (-0.8, 1.5) {$\texttt{l20}$};
   \node[rectangle,draw] (5) at (0.7, .7) {$\texttt{fUnl}$};
   \node[rectangle,draw] (6) at (3.6, -0) {$\texttt{ORPO}\,\,\checkmark$};
   \node[rectangle,draw] (7) at (7.5, -1.5) {$\texttt{CPO}\,\,\checkmark$};
   \node[rectangle,draw] (40) at (9.5, -1.5) {$\texttt{l14}$};
   \node[rectangle,draw] (8) at (7.8,-0.3) {$\texttt{cCPO}$};
   \node[rectangle,draw] (9) at (10,-0.2) {$\texttt{unCPO}$};

   \node[rectangle,draw] (10) at (4, 1.7) {$\texttt{qfUnl}$};
   \node[rectangle,draw] (60) at (4, -1.8) {$\texttt{l5}$};
   \node[rectangle,draw] (50) at (6.5, 1.5) {$\texttt{bCPO}$};
   \node[rectangle,draw] (11) at (6.8, .4) {$\texttt{cfUnl}$};

   \draw[->,thick] (0) edge[bend right=15] (1);
   \draw[->,thick] (0) -- (3);
   \draw[->,thick] (5) edge[bend left=10] (40);
   \draw[->,thick] (60) edge[bend right=10] (40);
   \draw[->,thick] (60) edge[bend right=10] (50);
   \draw[->,thick] (1) -- (60);
   \draw[->,thick] (10) -- (50);
   \draw[->,thick] (50) -- (9);
   \draw[->,thick] (0) -- (4);
   \draw[->,thick] (4) -- (10);
   \draw[->,thick] (2) edge[bend right=9] (50);
   \draw[->,thick] (4) edge[bend left=26] (60);
   \draw[->,thick] (7) -- (40);
   \draw[->,thick] (40) -- (9);
   \draw[->,thick] (20) -- (10);

   \draw[->,thick] (0) edge[bend right=20] (20);
   \draw[->,thick] (20) -- (2);
   \draw[->,thick] (20) -- (6);
   \draw[->,thick] (1) -- (7);
   \draw[->,thick] (1) -- (2);
   \draw[->,thick] (2) -- (8);
   \draw[->,thick] (4) -- (5);

   \draw[azure,thick] ($(1.north west)+(-2.9,0.5)$)  rectangle ($(2.south east)+(0.7,-0.7)$); 
   \draw[azure,thick] ($(4.south east)+(-2.9,-1.8)$)  rectangle ($(5.north east)+(0.5,1.8)$); 
   \draw[azure,thick] ($(7.south west)+(-5,-1.5)$)  rectangle ($(9.north east)+(0.3,2.7)$); 
   \draw[azure,thick] ($(0.south west)+(-1.4,-1.3)$)  rectangle ($(0.north east)+(0.2,1.5)$);

   \node[main_node,fill=azure!20] (20) at (6.5,3.1) {\footnotesize $\lmpredicate{M}(\xvar,\yvar_{l}) \to \lmpredicate{M}(\xvar,\yvar_{w})$};
   \node[main_node,fill=azure!20] (20) at (-0.8,3.1) {\footnotesize $\neg\lmpredicate{M}(\xvar,\yvar_{l})$};
   \node[main_node,fill=azure!20] (20) at (-5.2,-1) {\footnotesize $\lmpredicate{M}(\xvar,\yvar_{w})$};
   \node[main_node,fill=azure!20] (20) at (-5.1,3.1) {\footnotesize $\lmpredicate{M}(\xvar,\yvar_{w}) \land \neg \lmpredicate{M}(\xvar,\yvar_{l})$};
   
   \node[main_node] (20) at (9,-2.6) {\scriptsize $\lmpredicate{Ref}(\xvar,\yvar_{w}) \land \lmpredicate{M}(\xvar,\yvar_{l})$};
   \node[main_node] (20) at (9.2,-2.9) {\scriptsize $\to \lmpredicate{M}(\xvar,\yvar_{w}) \land \lmpredicate{Ref}(\xvar,\yvar_{l})$};

   \node[main_node] (20) at (4.2,-2.3) {\scriptsize $\lmpredicate{Ref}(\xvar,\yvar_{w}) \land (\lmpredicate{M}(\xvar,\yvar_{l}) \lor \neg\lmpredicate{Ref}(\xvar,\yvar_l))$};
   \node[main_node] (20) at (5,-2.6) {\scriptsize $\land (\neg\lmpredicate{M}(\xvar,\yvar_{w}) \lor \neg\lmpredicate{Ref}(\xvar,\yvar_l))$};
   \node[main_node] (20) at (3.5,-2.9) {\scriptsize $\to \lmpredicate{M}(\xvar,\yvar_{w}) \land \neg\lmpredicate{M}(\xvar,\yvar_l)$};
   
   \node[main_node] (20) at (9.5,0.85) {\scriptsize $\lmpredicate{M}(\xvar,\yvar_{l}) \land \lmpredicate{Ref}(\xvar,\yvar_w)$};
   \node[main_node] (20) at (9.5,0.6) {\scriptsize $\to \lmpredicate{M}(\xvar,\yvar_w)$};
   \node[main_node] (20) at (8.5,2.3) {\scriptsize $\lmpredicate{Ref}(\xvar,\yvar_{w}) \land (\lmpredicate{M}(\xvar,\yvar_{l}) \lor \neg\lmpredicate{Ref}(\xvar,\yvar_{l}))$};  
   \node[main_node] (20) at (8.3,2) {\scriptsize $\to \lmpredicate{M}(\xvar,\yvar_{w})$}; 
   \node[main_node] (20) at (-0.7,2.3) {\scriptsize $\lmpredicate{Ref}(\xvar,\yvar_{w}) \land (\lmpredicate{M}(\xvar,\yvar_{l}) \lor \neg\lmpredicate{Ref}(\xvar,\yvar_{l}))$};
   
   \node[main_node] (20) at (-1.25,2) {\scriptsize $\to \lmpredicate{M}(\xvar,\yvar_{w}) \land \neg\lmpredicate{M}(\xvar,\yvar_{l})$};
   \node[main_node] (20) at (-1.5,0.8) {\scriptsize $\lmpredicate{Ref}(\xvar,\yvar_{w}) \to \neg\lmpredicate{M}(\xvar,\yvar_l)$};
   \node[main_node] (20) at (0.2,-2.9) {\scriptsize $\lmpredicate{Ref}(\xvar,\yvar_{w}) \to \lmpredicate{M}(\xvar,\yvar_w)$};

   \node[main_node] (20) at (-5.9,2.3) {\scriptsize $\lmpredicate{Ref}(\xvar,\yvar_{w}) $};
   \node[main_node] (20) at (-5,2) {\scriptsize $\to \lmpredicate{M}(\xvar,\yvar_{w}) \land \neg \lmpredicate{M}(\xvar,\yvar_{l})$};

   \node[main_node] (20) at (-4.2,-2.6) {\scriptsize $\lmpredicate{Ref}(\xvar,\yvar_{w}) \land (\neg \lmpredicate{M}(\xvar,\yvar_{w}) \lor \neg \lmpredicate{Ref}(\xvar,\yvar_{l}))$};
   \node[main_node] (20) at (-5,-2.9) {\scriptsize $\to \lmpredicate{M}(\xvar,\yvar_{w}) \land \neg\lmpredicate{M}(\xvar,\yvar_{l})$};

   \draw[->,thick] (3) -- (5);

   \draw[->,thick] (3) -- (6);
   \draw[->,thick] (3) -- (7);
   \draw[->,thick] (7) -- (8);
   \draw[->,thick] (6) -- (8);
   \draw[->,thick] (8) -- (9);
   \draw[->,thick] (10) edge[bend left=15] (11);
   \draw[->,thick] (6) -- (11);
   \draw[->,thick] (5) -- (11);
   \draw[->,thick] (11) -- (9);

\end{tikzpicture}

\caption{\emph{What are interesting \texttt{DPO} variants to explore?} Extending the loss lattice in Figure~\ref{fig:lattice} to a version of the single model losses with reference models (i.e., their \textbf{reference forms}), showing different (largely unexplored) variants of \texttt{DPO} and the different semantics regions (gray boxes, corresponding to the core semantic formula for $\progvariable$ each set of losses). See Appendix~\ref{sec:new_losses} for details.}
\label{fig:reference_lattice}
\end{figure*}

\paragraph{There are many losses still to explore and we can exhaustively enumerate them.} We  created new losses by modifying the conditioning constraints of existing losses. Figure~\ref{fig:lattice} shows a (non-exhaustive) lattice representation of the loss landscape for single model preference approaches created by mechanically deriving new losses from the $\ell_{\texttt{CEUnl}}$ baseline (the most constrained) and ordering them by strict entailment (terminating in $\ell_{\texttt{unCPO}}$ \textcolor{black}{, our running example}).  We see different \textbf{semantic regions} emerge characterized by different formulas $\progvariable$, notably an unexplored region of unlikelihood losses ($\ell_{\texttt{l20}},\ell_{\texttt{cUNL}},\ell_{\texttt{fUNL}}$) that optimize for the negation of the loser $\neg \lmpredicate{M}(\xvar,\yvar_{l})$. Through compilation, any of these losses are now subject to experimentation. 


\textcolor{black}{
In Figure~\ref{fig:reference_lattice}, we show an extended version of Figure~\ref{fig:lattice} with the reference forms (gray boxes) of all losses (discussed below). Keeping the variables constant, this version exhaustively captures all definable \textcolor{black}{non-trivial} single model losses and preference structures $\overline{\progvariable}$ (nodes in this graph) that lie semantically in-between the baseline $\ell_{\texttt{CEUnl}}$ and $\ell_{\texttt{unCPO}}$, or formally all $\overline{\progvariable}$ s.t.\  $\overline{\progvariable}_{\texttt{CEUnl}} \sqsubseteq \overline{\progvariable} \sqsubseteq \overline{\progvariable}_{\texttt{unCPO}}$. We conjecture that this class of losses captures the most promising single model alternatives to the known losses $\ell_{\texttt{CPO}}$ and $\ell_{\texttt{ORPO}}$ and empirically investigate some of these losses below.
}

\paragraph{Adding a reference model has a clear, though sometimes peculiar, semantics.}  \textcolor{black}{\emph{What happens semantically when we add a reference model or term to our loss?}} As discussed in Section~\ref{sec:dpa} and Obs.~\ref{obs:reference_form}, specific losses can be made into a \texttt{DPO}-style losses with reference information (i.e., the reference form of that loss) by subtracting the log ratio $s_{\text{ref}}(y_{w},y_{l})$ from that loss's core loss equation. The following proposition shows the semantic result of adding reference information and follows directly from Obs.~\ref{obs:reference_form} and the application of Algorithm~\ref{alg:sl_translation}. 
\begin{prop}[semantics of reference forms] 
    Given a loss characterized by the core loss equation $\rho_{\theta}$ equal to  $\log \rho_{\theta}^{t} / \rho_{\theta}^{b}$, the core semantic formula $\progvariable$ for that loss's reference form is logically equivalent to the formula $(\textsc{Sem}(\rho_{\theta}^{b}) \land \lmpredicate{Ref}(\xvar,\yvar_{w})) \to (\textsc{Sem}(\rho_{\theta}^{t}) \land \lmpredicate{Ref}(\xvar,\yvar_{l}))$. 
\label{prop:semantics_ref_form}
\end{prop}

The semantics of \texttt{DPO}, which is the reference form of \texttt{CPO}, is shown in Table~\ref{tab:formalization} and is logically equivalent to a conjunction of two implications: $\lmpredicate{Ref}(\xvar,\yvar_{w}) \land \lmpredicate{M}(\xvar,\yvar_{l}) \to \lmpredicate{M}(\xvar,\yvar_{w})$ and $\lmpredicate{Ref}(\xvar,\yvar_{w}) \land \neg \lmpredicate{Ref}(\xvar,\yvar_{l}) \to \neg \lmpredicate{M}(\xvar,\yvar_{l})$. The first says that \emph{If the reference deems the winner to be valid and the tunable model deems the loser to be valid, then that model should also deem the winner to be valid}, while the second says that \emph{the tunable model should deem the loser to be not valid whenever the reference deems the winner to be valid and the loser to be not valid}. While this semantics makes sense, and complements nicely the semantics of \texttt{CPO} by adding information about the referent model, \texttt{DPO} includes conditioning constraints that are hard to justify from first principles, and that make it semantically disconnected from the \texttt{CE} and \texttt{CEUnl} baselines.

We also note that variants like \texttt{SimPO} and \texttt{DPOP} when formalized maintain exactly the same structure of \texttt{DPO} in Table~\ref{tab:formalization}, with \texttt{DPOP} adding repeated variables that amplify the score of the winner (see Appendix~\ref{sec:dpop}).  Giving the semantic similarity between these variants and \texttt{DPO}, any small semantic change found in one would likely be useful in these others, which motivates general exploration into varying the conditioning constraints. Several such variants of \texttt{DPO} and \texttt{SimPO} are shown in Figure~\ref{fig:reference_lattice} (i.e., the gray regions).

\begin{example}[Novel variant of cross-entropy]
In addition to finding novel variants of \texttt{DPO}, the reference forms can also reveal intriguing variants of standard losses like cross-entropy $\ell_{\texttt{CE}}$. Semantically, $\ell_{\text{CE}}$ can be expressed in an implication form as $\neg\lmpredicate{M}(\xvar,\yvar_{w}) \to \lmpredicate{M}(\xvar,\yvar_{w})$. By adding a reference model according to Prop.~\ref{prop:semantics_ref_form}, this results in the formula $(\neg\lmpredicate{M}(\xvar,\yvar_{w}) \land \lmpredicate{Ref}(\xvar,\yvar_{w})) \to \lmpredicate{M}(\xvar,\yvar_{w}) \land \lmpredicate{Ref}(\xvar,\yvar_{l})$, which simplifies to the logically equivalent formula $\lmpredicate{Ref}(\xvar,\yvar_{w}) \to \lmpredicate{M}(\xvar,\yvar_{w})$. Semantically, this leads to a variant of cross-entropy where updates are made based on signal from the reference model, which seems like a natural variation. The full loss equation results in $\rho_{\theta} = \log \frac{P_{\theta}(y_{w} \mid x)P_{\text{ref}}(y_{l} \mid x)}{ (1  - P_{\theta}(y_{w} \mid x))P_{\text{ref}}(y_{w} \mid x)}$; by modifying the conditioning constraints one can  arrive at different variants of this loss and directly implement each variant for experimentation.
\end{example}

\paragraph{Can we find empirically improved losses using our method?} The ultimate goal of our analysis is to facilitate the discovery of empirically improved DPA losses. As a case study, we implemented single model losses around the known $\ell_{\texttt{CPO}}$ in Figure~\ref{fig:lattice}, treating it as a baseline to improve upon. Using a model-as-judge style evaluation from \citet{hong2024reference} and a \texttt{Qwen-0.5B} LLM (details in Appendix~\ref{sec:experiment_details}), we found one particular loss, $\ell_{\texttt{cCPO}}$ to be competitive with $\ell_{\texttt{CPO}}$, achieving a win-rate of $52.0$ as shown in Table~\ref{tab:win_rate}. We also observe that different losses have markedly different performance across different datasets, suggesting that a one-size-fits-all approach isn't ideal---semantically different tasks are best learned using different losses.

\begin{table}[ht]
\centering 
\resizebox{\linewidth}{!}{
{\footnotesize
\begin{tabular}{| c c | c c c c c |}
\hline 
\textbf{loss} & WR\% ($\ell_{\text{cpo}}$) & evol & false-qa & flan & sharegpt & ultrachat \\ \hline 
$\ell_{\texttt{cfUNL}}$ & 46.1 ($\pm$\textcolor{gray}{0.4}) & 46.1 ($\pm$\textcolor{gray}{2.2}) & 51.6 ($\pm$\textcolor{gray}{2.9}) & 46.4 ($\pm$\textcolor{gray}{1.7}) & 46.2 ($\pm$\textcolor{gray}{1.2}) & 44.1 ($\pm$\textcolor{gray}{1.0}) \\
$\ell_{\texttt{qfUNL}}$ & 48.9 ($\pm$\textcolor{gray}{0.8}) & 45.3 ($\pm$\textcolor{gray}{1.9}) & \colorbox{gray!20}{34.7} ($\pm$\textcolor{gray}{6.3}) & \colorbox{gray!20}{57.9} ($\pm$\textcolor{gray}{1.2}) & 46.8 ($\pm$\textcolor{gray}{2.4}) & \colorbox{gray!20}{41.3} ($\pm$\textcolor{gray}{1.4}) \\
$\ell_{\texttt{cCPO}}$ & \colorbox{gray!20}{52.0} ($\pm$\textcolor{gray}{0.6}) & 50.7 ($\pm$\textcolor{gray}{0.5}) & 50.2 ($\pm$\textcolor{gray}{0.7}) & \colorbox{gray!20}{57.2} ($\pm$\textcolor{gray}{1.1}) & 47.2 ($\pm$\textcolor{gray}{1.8}) & 53.1 ($\pm$\textcolor{gray}{1.9}) \\
$\ell_{\texttt{unCPO}}$ & 46.0 ($\pm$\textcolor{gray}{0.2}) & 45.8 ($\pm$\textcolor{gray}{0.3}) & 52.1 ($\pm$\textcolor{gray}{3.0}) & 45.7 ($\pm$\textcolor{gray}{0.6}) & 46.2 ($\pm$\textcolor{gray}{2.1}) & 44.8 ($\pm$\textcolor{gray}{2.1}) \\
\hline 
\end{tabular}
}}
\caption{Results of a feasibility study involving \texttt{Qwen-0.5B} tuned on the new losses (rows) compared against the known loss $\ell_{\texttt{CPO}}$ (second column) on \texttt{ultrafeedback} test in aggregate (2nd column) and on subsets (right columns). See details in Section~\ref{sec:experiment_details}.}
\label{tab:win_rate}
\end{table} 

While small scale, this study demonstrates the feasibility of using our framework to derive empirically successful losses. Appendix~\ref{sec:experiment_details} reports additional experiments and findings.

\section{Conclusion}

Despite the routine use of a variety of DPA algorithms to align LLMs with human preferences, knowing what exactly the losses underlying these algorithms capture and how they relate to each other remains largely unknown. We presented a new technique for characterizing the semantics of such losses in terms of logical formulas over boolean propositions that capture model predictions. Key to our approach is a \emph{decompilation} procedure, allowing one to compositionally derive provably correct symbolic formulas corresponding to any loss function expressed as a ratio of disjoint multilinear polynomials. Our approach provides a fresh perspective into preference losses, identifying a rich loss landscape and opening up new ways for practitioners to explore new losses by systematically varying the symbolic formulas corresponding to existing successful loss functions.

\section*{Acknowledgements}

Special thanks to the following people for their feedback at various stages of the work (in alphabetical order): Kareem Ahmed, Gregor Betz, Junyan Cheng, Hamish Ivison, Maryna Kavalenka, Emile van Krieken, Nathan Lambert, Robin Manhaeve, Valentina Pyatkin, Antonio Vergari and Gijs Wijnholds, as well as the four anonymous reviewers who read an earlier draft (in particular, we thank one reviewer who provided useful details about standard WMC encodings and feedback related to Assumption~\ref{build_assumption_second}).  Thanks also to the audiences at Ai2, the University of Utah and the University of Stuttgart for listening to a talk version of this paper and for providing helpful feedback and discussion. As usual, any remaining mistakes remain our own.

\bibliography{cited}
\bibliographystyle{icml2025}


\appendix


\begin{table*}
\begin{center}
\setlength\arrayrulewidth{1.2pt}
 \resizebox{\linewidth}{!}{%
{\footnotesize
    \begin{tabular}{| l l c c l |}
        \hline 
        {\textbf{Loss name}}  & \multicolumn{1}{c}{\textbf{core loss equation} $\rho_{\theta}$} & {\textbf{CE term}} & \multicolumn{1}{c}{{\textbf{length norm.}}} & \multicolumn{1}{c|}{{\textbf{Extra details and terms}}} \\ 
        \hline 
        \multicolumn{5}{|c|}{{\textbf{common baseline losses}}} \\ \hline 
        $\ell_{\texttt{CE}}$ & $\log \frac{P_{\theta}(y_{w} \mid x)}{1 - P_{\theta}(y_{w} \mid x)}$ & -- & -- & \\ 
        $\ell_{\texttt{CEUnl}}$ {\tiny \citep{rafailov2023direct}} & $\log \frac{P_{\theta}(y_{w} \mid x)  (1 - P_{\theta}(y_{l} \mid x))}{1 - (P_{\theta}(y_{w} \mid x)  (1 - P_{\theta}(y_{l} \mid x)))}$ & -- &  -- & Unlikelihood term weighted by $\alpha$ \\ \hline 
        \multicolumn{5}{|c|}{{\textbf{reference approaches}}} \\  \hline
        $\ell_{\texttt{DPO}}$ {\tiny \citep{rafailov2023direct}}  & $\log \frac{P_{\theta}(y_{w} \mid x) \textcolor{azure}{P_{\text{ref}}(y_{l} \mid x)}}{ \textcolor{azure}{P_{\text{ref}}(y_{w} \mid x)} P_{\theta}(y_{l} \mid x)}$ & $\tikzxmark$ & $\tikzxmark$ &  \\ 
        $\ell_{\texttt{ODPO}}$ {\tiny \citep{amini2024direct}} & $\log \frac{P_{\theta}(y_{w} \mid x) P_{\text{ref}}(y_{l} \mid x)}{P_{\text{ref}}(y_{w} \mid x)P_{\theta}(y_{l} \mid x)}  - \textcolor{azure}{\gamma_{\text{offset}}}$ & $\tikzxmark$ & $\tikzxmark$ & Added offset term $\gamma_{\text{offset}}$ \\ 
        $\ell_{\texttt{DPOP}}$ {\tiny \citep{pal2024smaug}} & $\log \frac{P_{\theta}(y_{w} \mid x) \textcolor{azure}{P_{\theta2}(y_{w} \mid x)} P_{\text{ref}}(y_{l} \mid x)}{P_{\text{ref}}(y_{w} \mid x) \textcolor{azure}{P_{\text{ref2}}(y_{w} \mid x)} P_{\theta}(y_{l} \mid x)}$ & $\tikzxmark$ & $\tikzxmark$ & See Appendix~\ref{sec:dpop} \\ 
        $\ell_{\texttt{R-DPO}}$ {\tiny \cite{park2024disentangling}} & $\log \frac{P_{\theta}(y_{w} \mid x) P_{\text{ref}}(y_{l} \mid x)}{P_{\text{ref}}(y_{w} \mid x)P_{\theta}(y_{l} \mid x)}  + \textcolor{azure}{\gamma_{\text{len}}}$  & $\tikzxmark$ & $\tikzxmark$ & Added length bias term $\gamma_{\text{len}}$ \\
        $\ell_{\texttt{DPKD}}$ {\tiny \citep{dpkd}}  & $\log \frac{P_{\text{student}}(y_{w} \mid x) \textcolor{azure}{P_{\text{teacher}}(y_{l} \mid x)}}{ \textcolor{azure}{P_{\text{teacher}}(y_{w} \mid x)} P_{\text{student}}(y_{l} \mid x)}$ &  $\checkmark$ & $\checkmark$ & Distillation, re-parameterizes \texttt{ref} and $\theta$ \\  
        \hline
        \multicolumn{5}{|c|}{{\textbf{single model (no reference)}, CE weight $\lambda$}} \\ \hline 
        $\ell_{\texttt{CPO}}$ {\tiny \citep{xu2024contrastive}}  &   $\log \frac{P_{\theta}(y_{w} \mid x)}{P_{\theta}(y_{l} \mid x)}$ & $\checkmark$ & $\tikzxmark$ & Removes \texttt{ref} \\
        $\ell_{\texttt{ORPO}}$ {\tiny \citep{hong2024reference}}  &    $\log \frac{P_{\theta}(y_{w} \mid x)\textcolor{azure}{(1-P_{\theta}(y_{l} \mid x))} }{P_{\theta}(y_{l} \mid x) \textcolor{azure}{(1  - P_{\theta}(y_{w} \mid x))} }$ & $\checkmark$ & $\checkmark$ & $\beta = 1$, main loss weighted by $\alpha$, $\lambda=1$ \\ 
        $\ell_{\texttt{SimPO}}$ {\tiny \citep{meng2024simpo}}  & $\log \frac{P_{\theta}(y_{w} \mid x)}{P_{\theta}(y_{l} \mid x)} - \textcolor{azure}{\gamma}$   & $\tikzxmark$ & $\checkmark$ & Added margin term $\gamma$, re-formalized in Table~\ref{tab:comparison}   \\ \hline 
    \end{tabular}

}}
\end{center}
\caption{Details of the original losses from Table~\ref{tab:comparison} and others (adapted from \citet{meng2024simpo}), all of which were originally implemented using the logistic log-loss, i.e., each  $\ell_{x} = -\log \sigma(\beta \rho_{\theta})$. We also include details about whether cross-entropy regularization (\textbf{CE term}) and  length normalization (\textbf{length norm.}) were used (yes $\checkmark$, no $\times$) along with other details (\textbf{Extra details}) (e.g., extra weight terms, specific choices about $\beta$ or cross-entropy weight $\lambda$) that we either exclude or generalize in our analysis and experiments (e.g., extra loss weighting terms $\alpha$). See \citet{winata2024preference} for a comprehensive review and \citet{zhao2024rainbowpo} for an approach further mixes \texttt{DPO} and \texttt{SimPO}. 
}
\label{tab:original_loss_details}
\end{table*}

\section{Original losses}
\label{app:original_losses}

\textcolor{black}{Further details of the original losses in Table~\ref{tab:comparison}, along with other variants such as \texttt{R-DPO} \citep{park2024disentangling}, \texttt{ODPO} \citep{amini2024direct} and \texttt{DPKD} \citep{dpkd}, are shown in Table~\ref{tab:original_loss_details}. While our formalization abstracts over certain  details such as length normalization and additional regularization terms, we include such details from the original studies. In the case of regularization terms, as noted in Table~\ref{tab:original_loss_details} most \textbf{no reference} approaches add an additional cross-entropy term, often making the full losses in these studies equal to $\ell_{x+\texttt{CE}} = \ell_{x} + \lambda \ell_{\texttt{CE}}$ (with weight $\lambda$). In some cases, additional terms $\alpha$ are assumed that we abstract over in our analysis, e.g., in $\ell_{\texttt{orpo}}$ the full loss includes an additional weight term $\alpha$ that is added to the main loss (in our experiments below, $\alpha$ is implicitly set to $1$).}

\section{Compositionality constraint}
\label{app:compositionality}

\CompilationLimitations*   

\begin{proof}
   Taking $\ell_{\texttt{CPO}}$ as an example, the loss equation is based on the ratio $s_{\theta}(y_{w},y_{l})$ consisting of two predictions $P_{\theta}(y_{w} \mid x)$ and $P_{\theta}(y_{l} \mid x)$, which we can translate into the propositional formulas $\progvariable_{t} := \lmpredicate{M}(\xvar,\yvar_{w})$ and $\progvariable_{b} := \lmpredicate{M}(\xvar,\yvar_{l})$, consisting of a total of two atomic propositions. Translating this to the standard semantic loss involves finding a \emph{single} $\progvariable$ such that $\progvariable_{w} = \progvariable$ and $\progvariable_{l} = \neg \progvariable$. To see that no such $\progvariable$ exists, we can enumerate all 16 unique Boolean functions over variables $\lmpredicate{M}(\xvar,\yvar_{w})$ and $\lmpredicate{M}(\xvar,\yvar_{l})$ (the only variables we are allowed under Assumption~\ref{build_assumption_second}) and verify that none yield a single formula $\progvariable$ s.t.\ $\log \frac{\wmc{\progvariable}{}}{\wmc{\neg\progvariable}{}} = s_{\theta}(y_{w},y_{l})$. The same argument can be applied to each of the other non-baseline losses in the table. 
\end{proof}

Without the compositionality assumption, one can encode any $\rho_{\theta}$ as a formula using additional variables and weighting schemes, as is commonly done in standard WMC encodings \citep{chavira2008probabilistic}. However, the semantics of the resulting formulas are less transparent and often hidden in the weights. We instead propose to define below a novel (unweighted) encoding for preference that doesn't require additional variables, thereby facilitating a compositional and transparent translation from loss equations.


\begin{table}
\centering 
    \setlength\arrayrulewidth{1.2pt}
    \begin{tabular}{| c l |}
\hline 
 Input &  \multicolumn{1}{c|}{\textsc{sem}$(\cdot)$} \\ 
 \multicolumn{2}{|c|}{predictions} \\ 
$\probability_{\lmpredicate{M}}(\yvar \mid \xvar)$ & $\progvariable := \lmpredicate{M}(\xvar,\yvar)$ \\ \hline 
\multicolumn{2}{|c|}{formulas $\progvariable$}  \\ 
$\progvariable_{1} \cdot \progvariable_{2}$ & $\progvariable := \opand(\progvariable_{1},\progvariable_{2})$  \\ 
$1 - \progvariable$ &      $\progvariable :=\opneg(\progvariable)$  \\ 
$\progvariable_{1} + \progvariable_{2}$ &     $\progvariable := \opor(\progvariable_{1},\progvariable_{2})$  \\ \hline
\end{tabular}

\caption{Rules for the compositional translation of loss expressions into symbolic formulas. See again example in Figure~\ref{fig:orpo_derivation}.}
\label{tab:translation_rules}
\end{table}
\section{Semantic translation rules} 

In Table~\ref{tab:translation_rules} we show the full translation rules for Algorithm~\ref{alg:sl_translation}.

\begin{figure*}

\small
\centering 
\setlength\arrayrulewidth{1.2pt}
\setlength{\tabcolsep}{2.3pt}

\begin{tabular}{| c c | c c c c c c c c|} 
 \hline
$\lmpredicate{M}(\xvar,\yvar_{w})$ & $\lmpredicate{M}(\xvar,\yvar_{l})$ & $\ell_{ \texttt{ORPO} }$ & $\ell_{ \texttt{cUnl} }$ & $\ell_{ \texttt{l3} }$ & $\ell_{ \texttt{CEUnl} }$ & $\ell_{ \texttt{cCPO} }$ & $\ell_{ \texttt{CPO} }$ & $\ell_{ \texttt{CE} }$ & $\ell_{ \texttt{sCE} }$\\ \hline
T & T &  & \colorbox{red!10}{\tikzxmark} &  & \colorbox{red!10}{\tikzxmark} & \colorbox{gray!20}{\checkmark} & \colorbox{gray!20}{\checkmark}\colorbox{red!10}{\tikzxmark} & \colorbox{gray!20}{\checkmark} & \colorbox{gray!20}{\checkmark}\colorbox{red!10}{\tikzxmark} \\
T & F & \colorbox{gray!20}{\checkmark} & \colorbox{gray!20}{\checkmark} & \colorbox{gray!20}{\checkmark} & \colorbox{gray!20}{\checkmark} & \colorbox{gray!20}{\checkmark} & \colorbox{gray!20}{\checkmark} & \colorbox{gray!20}{\checkmark} & \colorbox{gray!20}{\checkmark} \\
F & T & \colorbox{red!10}{\tikzxmark} & \colorbox{red!10}{\tikzxmark} & \colorbox{red!10}{\tikzxmark} & \colorbox{red!10}{\tikzxmark} & \colorbox{red!10}{\tikzxmark} & \colorbox{red!10}{\tikzxmark} & \colorbox{red!10}{\tikzxmark} & \colorbox{red!10}{\tikzxmark} \\
F & F &  &  & \colorbox{red!10}{\tikzxmark} & \colorbox{red!10}{\tikzxmark} &  &  & \colorbox{red!10}{\tikzxmark} & \colorbox{red!10}{\tikzxmark} \\
\hline
\end{tabular}
\\[.2cm]

\begin{tabular}{| c c | c c c c c c c c|} 
 \hline
$\lmpredicate{M}(\xvar,\yvar_{w})$ & $\lmpredicate{M}(\xvar,\yvar_{l})$ & $\ell_{ \texttt{cfUnl} }$ & $\ell_{ \texttt{fUnl} }$ & $\ell_{ \texttt{qfUnl} }$ & $\ell_{ \texttt{l20} }$ & $\ell_{ \texttt{unCPO} }$ & $\ell_{ \texttt{l14} }$ & $\ell_{ \texttt{bCE} }$ & $\ell_{ \texttt{l5} }$\\ \hline
T & T &  & \colorbox{red!10}{\tikzxmark} &  & \colorbox{red!10}{\tikzxmark} & \colorbox{gray!20}{\checkmark} & \colorbox{gray!20}{\checkmark}\colorbox{red!10}{\tikzxmark} & \colorbox{gray!20}{\checkmark} & \colorbox{gray!20}{\checkmark}\colorbox{red!10}{\tikzxmark} \\
T & F & \colorbox{gray!20}{\checkmark} & \colorbox{gray!20}{\checkmark} & \colorbox{gray!20}{\checkmark} & \colorbox{gray!20}{\checkmark} & \colorbox{gray!20}{\checkmark} & \colorbox{gray!20}{\checkmark} & \colorbox{gray!20}{\checkmark} & \colorbox{gray!20}{\checkmark} \\
F & T & \colorbox{red!10}{\tikzxmark} & \colorbox{red!10}{\tikzxmark} & \colorbox{red!10}{\tikzxmark} & \colorbox{red!10}{\tikzxmark} & \colorbox{red!10}{\tikzxmark} & \colorbox{red!10}{\tikzxmark} & \colorbox{red!10}{\tikzxmark} & \colorbox{red!10}{\tikzxmark} \\
F & F & \colorbox{gray!20}{\checkmark} & \colorbox{gray!20}{\checkmark} & \colorbox{gray!20}{\checkmark}\colorbox{red!10}{\tikzxmark} & \colorbox{gray!20}{\checkmark}\colorbox{red!10}{\tikzxmark} & \colorbox{gray!20}{\checkmark} & \colorbox{gray!20}{\checkmark} & \colorbox{gray!20}{\checkmark}\colorbox{red!10}{\tikzxmark} & \colorbox{gray!20}{\checkmark}\colorbox{red!10}{\tikzxmark} \\
\hline
\end{tabular}

\caption{A Boolean representation (in the style of Figure~\ref{fig:boolean}) of the single model loss functions shown in Figure~\ref{fig:lattice}. See again Figure~\ref{fig:boolean} for how to interpret the corresponding losses.}
\label{fig:tabular}
\end{figure*}

\section{Proofs of propositions}
\label{sec:aux_proofs}

Below we state propositions discussed in Section~\ref{sec:logic_formal} with their proofs. 

\Monotonicity*   

\begin{proof}
By the definition of preference entailment, we have $\overline{\progvariable}_{f}^{(1)} \models \overline{\progvariable}_{f}^{(2)}$. This means that for any $d$, $\overline{\progvariable}^{1}(d) \models \overline{\progvariable}^{2}(d)$, which implies that for any $\theta$, $\wmc{\overline{\progvariable}^{(1)}(d)}{} \leq \wmc{\overline{\progvariable}^{(2)}(d)}{}$.
From the definition of preference entailment, we also have $\overline{\neg\progvariable}^{(2)}(d) \models \overline{\neg \progvariable}^{(1)}(d)$. Following a similar line of reasoning as above, this implies $\wmc{\overline{\neg\progvariable}^{(1)}(d)}{} \geq \wmc{\overline{\neg\progvariable}^{(2)}(d)}{}$.
Thus, for any $d$ and $\theta$, the weighted model counting ratio term in the semantic loss in Table~\ref{tab:sl_variants} is no larger for $\overline{\progvariable}^{(1)}$ than for $\overline{\progvariable}^{(2)}$. It follows that $\ell_{\text{sl}}(\overline{\progvariable}^{(1)},\theta,\{d\}) \geq \ell_{\text{sl}}(\overline{\progvariable}^{(2)},\theta,\{d\})$. Taking the expectation over $d \sim D$, we obtain $\ell_{\text{sl}}(\overline{\progvariable}^{(1)},\theta,D) \geq \ell_{\text{sl}}(\overline{\progvariable}^{(2)},\theta,D)$.
\end{proof}

It follows that equivalent preference structures have identical semantic losses:
\begin{cor}[semantic equivalence] 
    \label{cor:equivalence}
    If $\overline{\progvariable}^{1} \equiv \overline{\progvariable}^{2}$ then $\ell_{\text{sl}}(\overline{\progvariable}^{(1)},\theta,D) = \ell_{\text{sl}}(\overline{\progvariable}^{2},\theta,D)$ for any $\theta,D$. 
\end{cor}

The next result is an analogue to the locality property in the original semantic loss \citep{xu2018semantic}, which tells us that unused logical variables in formulas do not affect loss values, which allows us to compare losses with different number of variables.  

\begin{prop}[locality] Let $\overline{\progvariable}$ be a preference structure defined over probabilistic prediction variables $\mathbf{X}$ with parameters $\theta_{x}$. Let $\mathbf{Y}$ be some disjoint set of variables with parameters $\theta_{y}$. Then $\ell_{sl}(\overline{\progvariable},\theta_{x},D) = \ell_{sl}(\overline{\progvariable},[\theta_{x} \, \theta_{y}],D)$ for any $D$.
\label{prop:locality_appendix}
\end{prop}
\begin{proof}
Let $\textbf{w}_x$ be any world over variables $\mathbf{X}$ and $\textbf{w}_y$ be any world over (disjoint) variables $\mathbf{Y}$. Let $\textbf{w}_{x,y}$ denote the joint world. By the standard semantic loss, the probability of the world $\textbf{w}_{x,y}$ in the $(\mathbf{X}, \mathbf{Y})$ space can be written as $\probability_{\theta_x, \theta_y}(\textbf{w}_{x,y}) = \prod_{X_i \in \mathbf{X}} Q_{\theta_x,\theta_y}(X_i) \cdot \prod_{Y_j \in \mathbf{Y}} Q_{\theta_x,\theta_y}(Y_j)$ where $Q$ is either $\probability$ or $1 - \probability$. Since the parameters $\theta_x$ and $\theta_y$ refer to disjoint sets of variables, we can simplify this to $\prod_{X_i \in \mathbf{X}} Q_{\theta_x}(X_i) \cdot \prod_{Y_j \in \mathbf{Y}} Q_{\theta_y}(Y_j)$.

It follows that the marginal probability of the world $\textbf{w}_x$ in the $(\mathbf{X}, \mathbf{Y})$ space equals $\probability_{\theta_x,\theta_y}(\textbf{w}_x) = \sum_{\mathbf{Y}} \left(\prod_{X_i \in \mathbf{X}} Q_{\theta_x}(X_i) \cdot \prod_{Y_j \in \mathbf{Y}} Q_{\theta_y}(Y_j)\right) = \prod_{X_i \in \mathbf{X}} Q_{\theta_x}(X_i) \cdot \sum_{\mathbf{Y}} \left( \prod_{Y_j \in \mathbf{Y}} Q_{\theta_y}(Y_j) \right) = \prod_{X_i \in \mathbf{X}} Q_{\theta_x}(X_i) \cdot \prod_{Y_j \in \mathbf{Y}} \left( Q_{\theta_y}(Y_j) + (1 - Q_{\theta_y}(Y_j)) \right) = \prod_{X_i \in \mathbf{X}} Q_{\theta_x}(X_i) = \probability_{\theta_x}(\textbf{w}_x)$. This last expression is precisely the probability of the world $\textbf{w}_x$ in only the $\mathbf{X}$ space. Thus, $\probability_{\theta_x}(\textbf{w}_x) = \probability_{\theta_x,\theta_y}(\textbf{w}_x)$, which implies $\text{WMC}\big(\overline{\progvariable};\theta_x\big) = \text{WMC}\big(\overline{\progvariable};\theta_x,\theta_y\big)$ and similarly for $\overline{\neg \progvariable}$. From this, the claim follows immediately.
\end{proof}

\section{New losses in loss lattice}
\label{sec:new_losses}

To visualize the semantics of the single model losses shown in Figure~\ref{fig:lattice}, we use the Boolean truth table shown in Figure~\ref{fig:tabular}. As already illustrated in Figure~\ref{fig:boolean}, each loss column can be mechanically converted into a preference structure via the following steps: 1) translate \colorbox{gray!20}{$\checkmark$} and \colorbox{red!20}{$\times$} into two standard propositional formulas that are logically consistent with the marks, $\progvariable_{t}$ for $\progvariable_{b}$, respectively, then 2) apply the rules in Algorithm~\ref{alg:sl_translation}  to these formulas to get a preference structure $\overline{\progvariable}$. (Note that the formulas in boxes in Figure~\ref{fig:lattice} show the core formula $\progvariable$ in the resulting preference structure and intentionally hide details about the constraints.) 

\begin{figure*}
\centering 
\footnotesize
\setlength\arrayrulewidth{1.2pt}
\setlength{\tabcolsep}{2.3pt}
\begin{tabular}{| c c c c c c c c|} 
 \hline
$\lmpredicate{(M)Ref}(\xvar,\yvar_{w})$ & $\lmpredicate{M}(\xvar,\yvar_{l})$ & $\lmpredicate{(M)Ref}(\xvar,\yvar_{l})$ & $\lmpredicate{M}(\xvar,\yvar_{w})$ & $\ell_{\texttt{DPO/SimPO}}$ & $\ell_{\texttt{orpo-ref}}$ & $\ell_{\texttt{qfUNL-ref}}$ & $\ell_{\texttt{l5-ref}}$ \\ \hline
F & F & F & F &   &   &   & \colorbox{gray!20}{\checkmark} \\
F & F & F & T &   &   &   & \colorbox{gray!20}{\checkmark} \\
F & F & T & F &   &   & \colorbox{gray!20}{\checkmark} & \colorbox{gray!20}{\checkmark} \\
F & F & T & T & \colorbox{gray!20}{\checkmark} & \colorbox{gray!20}{\checkmark} & \colorbox{gray!20}{\checkmark} & \colorbox{gray!20}{\checkmark} \\
F & T & F & F &   &   &   &   \\
F & T & F & T &   &   &   & \colorbox{gray!20}{\checkmark} \\
F & T & T & F &   &   &   &   \\
F & T & T & T & \colorbox{gray!20}{\checkmark} &   &   & \colorbox{gray!20}{\checkmark} \\
T & F & F & F &   &   & \colorbox{red!10}{\tikzxmark} & \colorbox{gray!20}{\checkmark}\colorbox{red!10}{\tikzxmark} \\
T & F & F & T &   &   &   & \colorbox{gray!20}{\checkmark} \\
T & F & T & F &   &   & \colorbox{gray!20}{\checkmark}\colorbox{red!10}{\tikzxmark} & \colorbox{gray!20}{\checkmark}\colorbox{red!10}{\tikzxmark} \\
T & F & T & T & \colorbox{gray!20}{\checkmark} & \colorbox{gray!20}{\checkmark} & \colorbox{gray!20}{\checkmark} & \colorbox{gray!20}{\checkmark} \\
T & T & F & F & \colorbox{red!10}{\tikzxmark} & \colorbox{red!10}{\tikzxmark} & \colorbox{red!10}{\tikzxmark} & \colorbox{red!10}{\tikzxmark} \\
T & T & F & T & \colorbox{red!10}{\tikzxmark} &   &   & \colorbox{gray!20}{\checkmark}\colorbox{red!10}{\tikzxmark} \\
T & T & T & F & \colorbox{red!10}{\tikzxmark} & \colorbox{red!10}{\tikzxmark} & \colorbox{red!10}{\tikzxmark} & \colorbox{red!10}{\tikzxmark} \\
T & T & T & T & \colorbox{gray!20}{\checkmark}\colorbox{red!10}{\tikzxmark} &   &   & \colorbox{gray!20}{\checkmark}\colorbox{red!10}{\tikzxmark} \\
\hline
\end{tabular}

\caption{Boolean semantics of \texttt{DPO} and \texttt{SimPO} (column 5) and some novel variants of (columns 6-8) representing the different semantic regions in Figure~\ref{fig:reference_lattice}.}
\label{fig:reference_boolean}
\end{figure*}

With these preference structures, we can then obtain a compiled version of the loss by simply applying one of the versions of the semantic loss. In simplified terms, finding the compiled loss equation directly from a truth table for a given version of semantic loss with convex function $f$ (e.g., those listed in Table~\ref{tab:sl_variants}) involves the following
\begin{align*}
f \bigg( \log \frac{ \sum \colorbox{gray!20}{\checkmark} }{\sum {\colorbox{red!20}{$\times$}}} \bigg)
\end{align*}
where we can replace each $\sum .$ with the corresponding WMC equations for each mark, then simplify the resulting equation (i.e., the core loss equation) to arrive at a compact loss equation that can be directly used for implementation. 

\paragraph{Losses used in experiments} Employing the process above, below we show the core loss equations for the losses we used in our experiments in accordance with the form in Table~\ref{tab:comparison}: 
\begin{center}
 \setlength\arrayrulewidth{1.2pt}
\begin{tabular}{| c c |}
\hline 
\textbf{Loss name} & \textbf{Core loss equation} (implementation) \\ 
$\ell_{\texttt{cpo}}$ &  $\log \frac{P_{\theta}(y_{w} \mid x)}{P_{\theta}(y_{l} \mid x)}$  \\ 
$\ell_{\texttt{orpo}}$ & $\log \frac{P_{\theta}(y_{w} \mid x) (1 - p_{\theta}(y_{l} \mid x))}{P_{\theta}(y_{l} \mid x) (1 - p_{\theta}(y_{w} \mid x))}$ \\ \hline 
$\ell_{\texttt{cCPO}}$ & $\log \frac{ P_{\theta}(y_{w} \mid x) }{ (1 - P_{\theta}(y_{w} \mid x))P_{\theta}(y_{l} \mid x) }$ \\ 
$\ell_{\texttt{qfUNL}}$ & $\log \frac{(1 - P_{\theta}(y_{l} \mid x))}{(1 - P_{\theta}(y_{w} \mid x)}$ \\
$\ell_{\texttt{cfUNL}}$ & $\log \frac{(1 - P_{\theta}(y_{l} \mid x))}{(1 - P_{\theta}(y_{w} \mid x))P_{\theta}(y_{l} \mid x)}$ \\ 
$\ell_{\texttt{unCPO}}$ & $\log \frac{P_{\theta}(y_{l} \mid x)P_{\theta}(y_{w}\mid x) + (1 - P_{\theta}(y_{l} \mid x))}{P_{\theta}(y_{l} \mid x)(1 - P_{\theta}(y_{w} \mid x))}$ \\ \hline 

\end{tabular}
 
\end{center}
As described above, the final loss that we implemented was then obtained by applying the logistic loss loss over these equations and adding a $\beta$ term and cross-entropy terms (see details below). We used the \texttt{trl} library for implementation from \cite{vonwerra2022trl}, with assistance from the trainer scripts used in \citet{meng2024simpo}.\footnote{see \url{https://github.com/huggingface/trl} and \url{https://github.com/princeton-nlp/SimPO}.} 

\paragraph{Extending the loss lattice to reference models} While our loss lattice and the subsequent experiments we describe center around novel no reference loss functions, we note that given abstract structure of DPA, we can easily transform a no reference loss function into reference loss function by simply subtracting the reference log win-lose ratio, $s_{\text{ref}}(y_{w},y_{l})$ (either using a real reference ratio or one for simpo) from any single model loss equation (e.g., any of of the loss equations above). Via some algebraic simplification, we can then arrive a new core loss equation with this reference information and straightforwardly generate a preference structure via Algorithm~\ref{alg:sl_translation}. 

Figure~\ref{fig:reference_lattice} shows the result of this process for the single loss functions derived in Figure~\ref{fig:lattice}. This reveals a wide range of novel variants of \texttt{DPO} that we leave for future experiments and study. Figure~\ref{fig:reference_boolean} shows the Boolean semantics of \texttt{DPO}/\texttt{SimPO} and some novel variants based on the reference form of \texttt{ORPO} ($\ell_{\texttt{ORPO-ref}}$), \texttt{qfUNL} ($\ell_{\texttt{qfUNL-ref}}$) and \texttt{l5} ($\ell_{\texttt{l5-ref}}$).

\begin{figure}
\begin{tcblisting}{boxrule=2pt, listing only,listing options={language=iPython},title={}}
from sympy import * 
# winner (W), loser (L), 
#(ref) winner (R_w), loser (R_l)
W,L,R_w,R_l = symbols('W,L,R_w,R_l') 
## equation translation for ORPO 
P_t = And(W,Not(L))
P_b = And(L,Not(W)) 
## pref. structure `$\color{cyan}{\overline{\progvariable}} = (\progvariable,{\progvariable_{\textbf{C}}},{\progvariable_{\textbf{A}}})$'
P = Implies(P_b,P_t).simplify() 
assert P.equals(Implies(L,W))
P_C = Or(P_t,P_b).simplify()     
P_A = And(P_t,P_b).simplify()    
## The reference form formula 
P_ref = Implies( 
    And(P_b,R_w), And(P_t,R_l)
).simplify() 
assert P_ref.equals(
    Implies(And(R_w,L),W)
)
\end{tcblisting}
\caption{An example showing how to compute the simplified symbolic formulas in preference structures for \texttt{ORPO} (see Figure~\ref{fig:orpo_derivation}) in \texttt{Sympy} \citep{meurer2017sympy}. }
\label{fig:sympy_simplification}
\end{figure}

\paragraph{Computing preference structures} \textcolor{black}{Figure~\ref{fig:sympy_simplification} shows how to symbolically compute preference structure representations in Python using the computer algebra library \texttt{Sympy} \citep{meurer2017sympy}. Specifically, lines 8-12 show how to compute a preference structure in the no-reference case, and lines 14-20 show how to compute a reference form of \texttt{ORPO} by adding a reference ratio.}  

\section{Experiments and Case studies} 
\label{sec:experiment_details}


\textcolor{black}{Our formal analysis reveals that the space of DPA losses is large, yet structured in systematic ways that we can now describe through symbolic encodings. Through case studies involving the new losses in Figure~\ref{fig:lattice}, we discuss some empirical results that give tips for how to better navigate this space and look for improved DPA losses using our framework. Specifically, we focus on losses around the known loss $\ell_{\text{CPO}}$, which we treat as a natural baseline to compare against. All experiments are performed using a 0.5 billion parameter LLM, \texttt{Qwen-0.5B} \citep{bai2023qwen}, tuned using \texttt{trl} \citep{vonwerra2022trl} on the \texttt{ultrafeedback} dataset; following standard practice, losses were implemented with a weighted cross-entropy regularizer term.} 

\textcolor{black}{While these experiments are small scale and limited in scope, they are merely meant to suggest possible uses our framework and open questions. We also share some general observations and conjectures that we hope motivates future research in this area.}

Below we provide details of the experiment setting then discuss some results and observations.

\paragraph{Dataset and Model} Following much of the DPA work we cite, we train models on the \texttt{ultrafeedback} dataset \citep{cui2023ultrafeedback}, which contains around 60k binarized preference pairs aggregated from several individual preference datasets (the different categories are listed in Table~\ref{tab:win_rate}). For tuning (detailed below) we used a custom held-out development set containing around 1.3k examples taken from the train set and reserve the test set (containing 2k examples) for final evaluation. 

Standardly, we ran experiments starting from a instruction tuned model (SFT), using a \texttt{Qwen-0.5B} (containing .5 billion parameters) base model \citep{bai2023qwen} that was initially tuned on 6k pairs from the \texttt{deita} dataset of \cite{liu2023makes}. To avoid repeating the process of instruction tuning, we started from the trained \texttt{Qwen} model released in the TRL library\footnote{\url{https://huggingface.co/trl-lib/qwen1.5-0.5b-sft}}.  

\paragraph{Hyper-parameters and model selection}  The following are the standard set of tunable hyper-parameters involved in our experiments: the $\beta$ term for DPA losses (see again Table~\ref{tab:f}), the learning rate, number of epochs, batch size and length normalization. Following other studies, we also regularized our losses with cross-entropy terms (CE) that include a tunable weight parameter $\lambda$ that controls their contribution to the gradient. Specifically, we kept set $\beta$ to 1, and experimented with learning rates in the range \{\texttt{1e-6}, \texttt{3e-6}, \texttt{8e-6}, \texttt{9e-7}\}, number of epochs in the range of $\{3, 5, 8\}$ and  batches sizes in the range \{ \texttt{32}, \texttt{128} \} (for efficiency reasons, most tuning with done with a batch size of 32), which follow many of the suggested ranges in \citet{meng2024simpo}.  Importantly, length normalization was used throughout to make all losses comparable and given that it has been shown to improve training performance \citep{meng2024simpo}. We used $\lambda$s in the range of $\{0.0, 0.01, 0.1, 0.3, 1.0\}$ (we found lower values, around $0.01$ and $0.1$, to be most effective). 

For each loss function we searched the best hyper-parameters by performing a comprehensive grid search over the ranges detailed above. Final model selection was then performed by performing inference with each trained model on our held-out development set and scoring the resulting generating outputs using an off-the-shelf reward model, in particular, a 1.8B parameter reward model from \cite{cai2024internlm2}\footnote{\url{internlm/internlm2-1_8b-reward}}. We then selected the models with the highest average reward score over the development set for comparison.

For the log probability experiments shown in Figure~\ref{fig:constrainedness}, we kept the learning rate, epoch and cross-entropy term constant (with learning rate equal to $1e-6$, 3 epochs, and a low cross-entropy term $0.01$) to directly compare the different approaches and try to bring out their more extreme behavior.

\paragraph{Evaluation protocol and win-rate comparison} We compare models tuned using our different losses using a procedure similar to how model selection is performance, which also follows the setup in  \citet{hong2024reference}. Specifically, we do a instance-level comparison of the reward score given for each generated output, compare that score with the score of our baseline $\ell_{\texttt{cpo}}$ and compute an overall win-rate, i.e., \% of instances where the reward score is higher than or \underline{equal to} the reward score for $\ell_{\texttt{cpo}}$ (we consider cases where items are equal given that some tasks involve generating single token output, such as the identifier of a multiple choice question or \textbf{yes} or \textbf{no}). We report the average win-rate averaged over 3 runs of each models with different generation seeds.

\subsection{Results and discussion}
\label{app:results_discussion}

\paragraph{How does constrainedness relate to loss behavior? Unintentional alignment shortcuts}  \textcolor{black}{Moving left to the right in Figure~\ref{fig:lattice} yields semantically less constrained losses. For example, we see through the Boolean semantics in Figure~\ref{fig:constrainedness} that some unconstrained losses can be satisfied by making the winner and loser both false ($\ell_{\texttt{unCPO}}, \ell_{\texttt{cfUNL}}$) or by making the the winner and loser both true ($\ell_{\texttt{unCPO}}, \ell_{\texttt{cfUNL}}$). \textcolor{black}{One natural question is: \emph{How does constrainedness contribute to a loss functions empirical success?}}} 

\textcolor{black}{We observe, consistent with other recent work on neuro-symbolic modeling \citep{marconato2024not,van2024independence}, that such unconstrainedness can yield extreme behavior as illustrated in Figure~\ref{fig:constrainedness}.  For example, $\ell_{\texttt{unCPO}}$ and $\ell_{\texttt{cfUNL}}$ attempt to make both the winners and losers false by driving their probability in the direction of zero (as shown in in both training (b) and evaluation (c)), whereas $\ell_{\texttt{cfUNL}}$ keeps both probabilities high to make both true. When viewing learning as a constraint satisfaction problem, such behavior makes sense and could help to better understand various spurious training behavior observed elsewhere in the DPA literature, e.g., related to likelihood displacement and \emph{unintentional unalignment} studied in \citet{razin2024unintentional} or issues with preference ranking \citep{chen2024preference}. 
}

\textcolor{black}{These results suggest that understanding the way in which a loss is constrained and whether it gives rise to spurious or \textbf{unintentional alignment shortcuts}  (e.g., making both predictions false) is an important factor when designing new loss functions. We note that existing losses in Figure~\ref{fig:lattice} are in the middle of the two extreme points and seem less susceptible to such extreme behavior, which could explain their success.}

\begin{figure}

\centering 

\setlength{\tabcolsep}{3pt}
\begin{tabular}{l c}

{\footnotesize (\textbf{A})} & \\ 
\multicolumn{2}{c}{
    \includegraphics[scale=0.9]{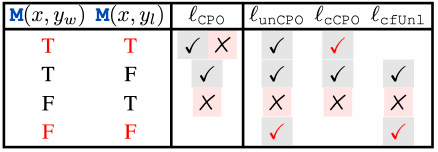}
} \\ 
{\footnotesize \textbf{(B)}} & \\[-0.05cm] 
\includegraphics[scale=0.28]{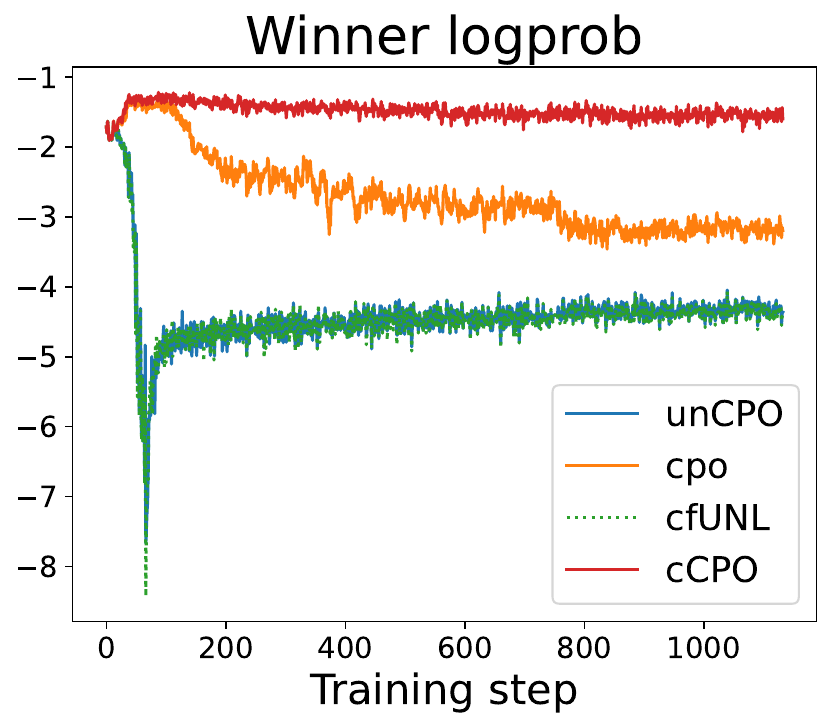} &  \includegraphics[scale=0.28]{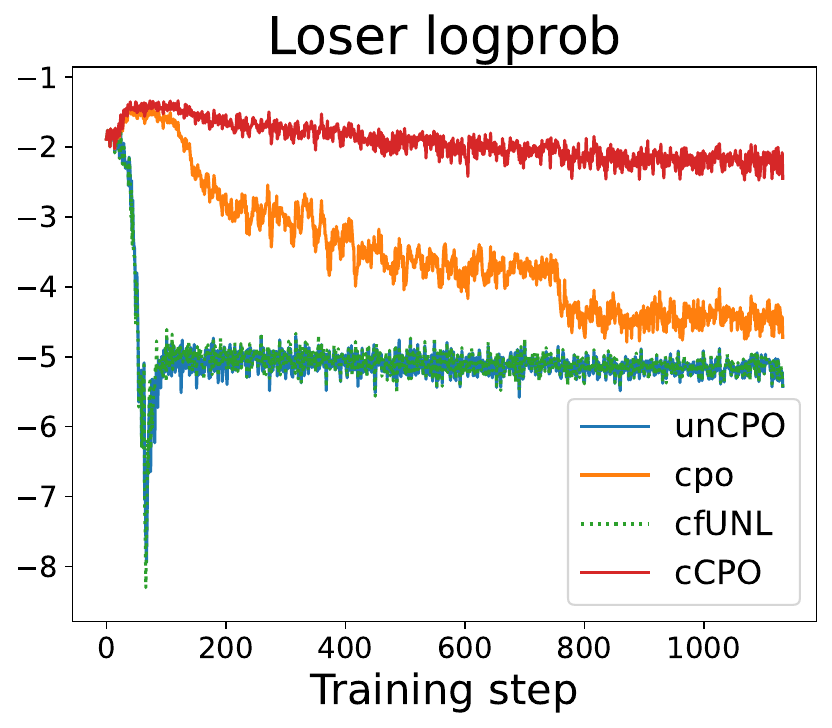} \\[-.2cm] 
{\footnotesize \textbf{(C)}} & \\ 
\includegraphics[scale=0.28]{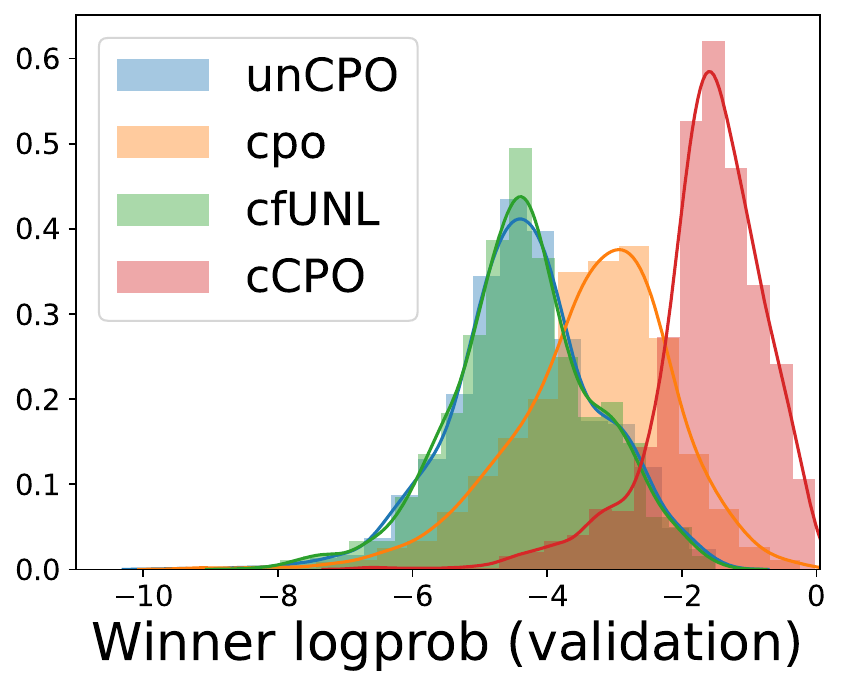} & \includegraphics[scale=0.28]{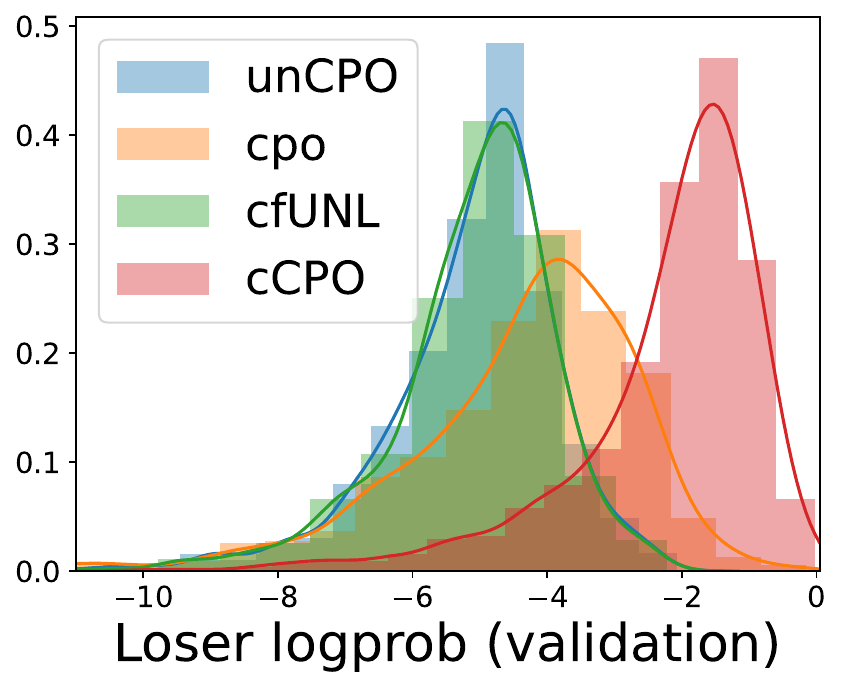} \\ 
\end{tabular}

\caption{An illustration (A) of how to semantically satisfy losses (\colorbox{gray!20}{$\checkmark$}) and the corresponding log probability behavior during training (B) and evaluation (C).}
\vspace{-0.6cm}
\label{fig:constrainedness}
\end{figure}

\paragraph{Can we find empirically improved losses using our method? Formalize and refine} \textcolor{black}{Our ultimate aim to use our framework to help discover new and successful preference algorithms.} \textcolor{black}{Given the spurious behavior of losses $\ell_{\texttt{unCPO}}$ and $\ell_{\texttt{cfUNL}}$, we would expect them to be less empirically successful. To test this and compare against $\ell_{\texttt{CPO}}$, we performed a model-as-judge-style experiment based on \cite{hong2024reference} that uses an off-the-shelf reward model \citep{cai2024internlm2} to score the outputs generated by our new models using the prompts from the  \texttt{ultrafeedback} test set. We then compare these rewards scores against those of $\ell_{\texttt{CPO}}$ to compute a win-rate, which gives an indication of improved or comparable generation quality over $\ell_{\texttt{CPO}}$. Indeed, we see in Table~\ref{tab:win_rate} that in aggregate, $\ell_{\texttt{unCPO}}$ and $\ell_{\texttt{cfUNL}}$ have the lowest win-rate against $\ell_{\texttt{CPO}}$. Interestingly, we see that $\ell_{\texttt{cCPO}}$ has a win-rate that suggests comparable generation quality to $\ell_{\texttt{CPO}}$, which shows the potential of using our framework to derive new and empirically successful losses.}

These experiments are an exercise in an approach we call \textbf{formalize and refine}, i.e., starting from empirically successful losses such as $\ell_{\texttt{CPO}}$, one can formalize such losses then modify the semantics to be more or less constrained based on empirical findings. We think more large scale exploration of the full loss space, especially for \texttt{DPO}, is a promising direction of future research.

\paragraph{Is there a single semantics for all preference learning? The different semantics conjecture} \textcolor{black}{We note that win-rate across different categories in \texttt{ultrafeedback} (i.e., the right most columns in Table~\ref{tab:win_rate}) varies quite considerably across models and loss types. This suggests that different types of preference data rely on a different semantics of preference, which requires a tuning approach that's tailored to those differences. We conjecture that such a phenomenon is likely to be wide spread across different tasks and datasets, and we see more empirical work on understanding the kinds of semantics needed in different scenarios as a promising direction of future work. Such work will benefit for recent attempts as incorporating more fine-grained annotation into preference such, such as in \citet{miranda2024hybrid}.}

\section{DPOP equation}
\label{sec:dpop}

The \texttt{DPOP} loss function in Table~\ref{tab:comparison} adds to the \texttt{DPO} an additional log term $\alpha \cdot \max(0,\log \frac{\probability_{\text{ref}}(y_{w} \mid x)}{\probability_{\theta}(y_{w} \mid x)})$ that aims to ensure that the log-likelihood of preferred example is high relative to the reference model (we simplified this loss by removing the $\max$ and $\alpha$ parameter, the latter of which is set to be a whole number ranging from 5 to 500 in \citet{pal2024smaug}). When translating the full loss into a single $\log$, this results in the equation $$\rho_{\theta} = \log \frac{
   \probability_{\text{ref}}( y_{l} \mid x)\probability_{\theta}(y_{w} \mid x)^{2}
 }{
   \probability_{\text{ref}}(y_{w} \mid x)^{2}\probability_{\theta}(y_{l} \mid x)}$$ for $\alpha=1$. The top and bottom equations are hence not multilinear since they both contain exponents $> 1$. To fix this, we  can simply create copies of these variables, e.g., with $P_{\theta}( y_p \mid x)^{2}$ and $P_{\text{ref}}( y_{l} \mid x)^{2}$ set to $P_{\theta}(y_p \mid x)P_{\theta2}(y_p \mid x)$ and $P_{\text{ref}}( y_{l} \mid x)P_{\text{ref2}}( y_{l} \mid x)$ using the copied prediction variables $P_{\theta2}(\cdot)$ and $P_{\text{ref2}}(\cdot)$. This type of variable copying also allows us to take into account the $\alpha$ and $\max$ above by setting the values of these copied variable to be $1$ whenever the log ratio is less than 0. 

Below we show the core semantic formula for \texttt{DPOP}, which, as noted before, makes a small adjustment to the \texttt{DPO} semantics as shown in Table~\ref{tab:formalization}:
\begin{lstlisting}
`$\progvariable :=$' Implies(
   And(Ref(`$\xvar$',`$\yvar_w$'),Ref`$_{2}$'(`$\xvar$',`$\yvar_w$'),M(`$\xvar$',`$\yvar_l$')), 
   And(Ref(`$\xvar$',`$\yvar_l$'),M(`$\xvar$',`$\yvar_w$'), M`$_{1}$'(`$\xvar$',`$\yvar_w$'))
)  
\end{lstlisting}

\section{Fuzzy derivations and semantics} 
\label{sec:fuzzy}

In contrast to the probabilistic logic approach pursued in our paper, real-valued fuzzy logics \citep{zadeh1975fuzzy} extend and relax classical logic by allowing truth values to have a continuous range. In these systems, traditional Boolean operators take the form of continuous functions, based on the theory of t-norms \citep{klement2013triangular}, which provides a means for directly translating logic into a differentiable form. As such, they have been widely used in machine learning as a way to integrate symbolic knowledge into learning \citep[\emph{inter alia}]{van2022analyzing,rocktaschel2015injecting,donadello2017logic,minervini2018adversarially,marra2019integrating}.

\begin{table} 
    \centering
    \setlength\arrayrulewidth{1.2pt}
    \begin{tabular}{| c c |}
\hline 
\textbf{Boolean logic} & \textbf{$\mathcal{R}$-Product} \\ \hline 
$\opand(a,b)$ & $\mathbf{a} \cdot \mathbf{b}$ \\ 
$\opneg(a)$ & $1 - \mathbf{a}$ \\ 
$\opor(a,b)$ & $\mathbf{a} + \mathbf{b} - \mathbf{a} \cdot \mathbf{b}$ \\
$\opimplication(a,b)$ & $\min(1,\frac{\mathbf{b}}{\mathbf{a}})$ \\ \hline 
\end{tabular}
  \caption{The translation of classical logic operators to $\mathcal{R}$-product logic for Boolean propositions $a,b$ and their relaxed versions $\mathbf{a},\mathbf{b}$.}
  \label{fig:rproduct}
\end{table}
\paragraph{No reference approach} For example, in Table~\ref{fig:rproduct} we define the semantics of the $\mathcal{R}$-product variant of fuzzy logic studied in \citet{grespan2021evaluating,giannini2023t}, which we use to derive fuzzy formulas for our preference losses. For convenience, we will use $\fuzzy{\progvariable}$ to denote the relaxed value of a formula $\progvariable$ under the semantics in Table~\ref{fig:rproduct} (for simplicity, we define this fuzzy function in terms of single formulas $\progvariable$ instead of preference structures). The fuzzy loss for $\progvariable$ is then defined as below: 
\begin{align}
\ell_{\text{fuzz}}(\progvariable,\theta,D) := \mathop{\mathbb{E}}\limits_{d \sim D} \bigg[ -\log \simplefuzzy{\progvariable_{d}} \bigg].
\end{align}

For the single model case, taking the core formulas to be $\progvariable$ to be the following (i.e., the core semantic formula for $\texttt{CPO}$):
\begin{lstlisting}
`$\progvariable_{\text{CPO}} :=$' Implies(M(`$\xvar$',`$\yvar_{l}$'), M(`$\xvar$',`$\yvar_{w}$'))
\end{lstlisting}
we see the following holds (via algebraic manipulation): 
\begin{align*}
-\log \simplefuzzy{\progvariable_\texttt{CPO}} &= -\log \min \bigg(1,  \frac{P_{\theta}(y_{w} \mid x)}{P_{\theta}(y_{l} \mid x)} \bigg) \\
&= \max\bigg(0,-\log \frac{P_{\theta}(y_{w} \mid x)}{P_{\theta}(y_{l} \mid x)} \bigg) 
\end{align*}
Making the $\ell_{\text{fuzz}}(\progvariable_{\text{CPO}},\theta,D)$ equal to perceptron-style loss \texttt{RRHF} \citep{yuan2023rrhf} in Table~\ref{tab:f}. Given that the same core semantic formula above can be recovered between the fuzzy and probabilistic approaches, we see this as giving additional motivation to using our preference structure representations. 


\paragraph{DPO and reference approaches} For \texttt{DPO} we see a similar derivation. Given the same core formula from Table~\ref{tab:formalization}: 
\begin{lstlisting}
`$\progvariable_{\texttt{DPO}} :=$' Implies(
  And(Ref(`$\xvar$',`$\yvar_{w}$'),M(`$\xvar$',`$\yvar_{l}$')), 
  And(Ref(`$\xvar$',`$\yvar_{l}$'),M(`$\xvar$',`$\yvar_{w}$'))
)
\end{lstlisting}
we see the following explicit derivation into fuzzy logic:
\begin{align*}
\simplefuzzy{\progvariable_{\texttt{DPO}}} =  \min\left(1, 
  \frac{
  P_{\text{ref}}\left( y_l \mid x \right)P_{\theta}\left( y_w \mid x \right)
}{
  P_{\text{ref}}\left( y_w \mid x \right)P_{\theta}\left( y_l \mid x \right)
}
\right) 
\end{align*}
and $-\log \,\simplefuzzy{\progvariable_{\texttt{DPO}}}$ equal to:
\begin{align*}
\max\left(0, -\bigg(
\log\frac{P_{\theta}\left( y_w \mid x \right)}{P_{\text{ref}}\left( y_w  \mid x \right)} -
\log\frac{P_{\theta}\left( y_l \mid x \right)}{P_{\text{ref}}\left( y_l  \mid x \right)} 
\right)\bigg)
\end{align*}
yielding once again a perceptron-style version of \texttt{DPO} $\ell_{\text{fuzz}}(\progvariable_{\text{DPO}},\theta,D)$ similar to the RRHF approach.

\paragraph{Derivation for Fuzzy Logic} To perform decompilation with fuzzy logic, we can employ a variant Algorithm~\ref{alg:sl_translation} that removes lines 4 and 5 and makes $\progvariable_{\textbf{C}} := \top$ and $\progvariable_{\textbf{A}} := \bot$ by default. Importantly, given the syntactic nature of fuzzy logic, whether or not \textsc{Simplify} is applied in line 3 will give rise to different fuzzy loss values since the fuzzy loss is not invariant to logical equivalence (see discussion in \citet{marra2024statistical}).

\end{document}